\documentclass{article} %
\usepackage{iclr2022_conference,times}

\usepackage[utf8]{inputenc} %
\usepackage[T1]{fontenc}    %
\usepackage{hyperref}       %
\usepackage{url}            %
\usepackage{booktabs}       %
\usepackage{amsfonts}       %
\usepackage{nicefrac}       %
\usepackage{microtype}      %
\usepackage{thm-restate}
\usepackage{listings}

\usepackage{xcolor}
\usepackage[]{hyperref}
\definecolor{cite_color}{HTML}{114083}
\definecolor{link_color}{RGB}{153, 0,0}  %
\definecolor{url_color}{RGB}{153, 102,  0}
\definecolor{emp_color}{RGB}{0,0,255}
\hypersetup{
 colorlinks,
 citecolor=cite_color,
 linkcolor=link_color,
 urlcolor=url_color}
\usepackage{graphicx}
\graphicspath{{./figs/}}

\usepackage{epsfig}
\usepackage{amsmath}
\usepackage{amsthm}
\usepackage{amssymb}
\usepackage{mathrsfs}
\usepackage{graphicx}
\usepackage{subfig}
\usepackage{multirow}
\usepackage{booktabs}
\usepackage{wrapfig}
\usepackage{enumerate}
\usepackage{bm}
\usepackage{bbm}
\usepackage{tabularx}

\usepackage{mathtools}

\usepackage[vlined, ruled, linesnumbered]{algorithm2e}

\SetCommentSty{mycommfont}
\SetKwInOut{Input}{input}
\SetKwInOut{Output}{output}

 \usepackage{cleveref} %
 \crefname{section}{Sec.}{Sections}
 \crefname{theorem}{Theorem}{Theorems}
 \crefname{lemma}{Lemma}{Lemmas}
 \crefname{equation}{Eq.}{Equations}
 \crefname{proposition}{Proposition}{Propositions}
 \crefname{claim}{Claim}{Claims}
\crefname{appendix}{Appendix}{Appendices}
\crefname{algorithm}{Alg.}{Algorithms}
 \crefname{figure}{Fig.}{Figures}
 \crefname{table}{Table}{Tables}
 \crefname{remark}{Remark}{Remarks}
 \crefname{definition}{Def.}{Definitions}
 \crefname{corollary}{Corollary}{Corollaries}

\allowdisplaybreaks

\usepackage{thmtools}
\usepackage{thm-restate}
\let \oldtextcircled \textcircled
\renewcommand{\textcircled}[1]{\oldtextcircled{\footnotesize #1}}

\usepackage{enumitem}
\setlist[itemize]{leftmargin=9mm}

\newcommand{\appendixtitle}[1]{
	\begin{center}
		\LARGE \bf #1
	\end{center}
}

\def\E{{\mathbb E}}

\def\M{{\cal M}}

\def \w{\mathbf{w}}

\def \x{\mathbf{x}}

\def \s{\mathbf{s}}

\def \BS{\mathbf{S}}

\def \BW{\mathbf{W}}

\def \R{{\mathbb{R}}}
\def \trans{\top}

\newcommand{\pare}[1]{{#1}}  %

\def \E{{\mathbb{E}}} %

\newcommand{\argmin}{{\arg\min}}

\newcommand{\algname}[1]{{\textsc{#1}}}

\newcommand{\fracpartial}[2]{\frac{\partial #1}{\partial #2}}

\newcommand{\bas}{\mathbf{e}} %
\def \chara{\mathbf{e}} %

\newcommand{\groundset}{\ensuremath{{N}}}

\newcommand{\sete}[3]{\mathbf #1 | #1_{#2}\!\gets\!#3}

\newtheorem{proposition}{Proposition}

\newtheorem{definition}{Definition}
\newtheorem{remark}{Remark}

\newcommand{\parti}{\text{Z}} %
\newcommand{\multi}{f^F_{\text{mt}}} %

\newcommand{\shapley}{\ensuremath{\text {Sh}}}
\newcommand{\banzhaf}{\ensuremath{\text {Ba}}}
\newcommand{\varindex}{\text{Variational Index}\xspace} %
\newcommand{\mfi}{{\text{MFI}}\xspace} %

\newcommand{\kl}[2]{\mathbb{KL}(#1\|#2)}
\newcommand{\entropy}[1]{\mathbb{H}(#1)}

\usepackage{etoc}
\etocdepthtag.toc{mtchapter}
\etocsettagdepth{mtchapter}{subsection}
\etocsettagdepth{mtappendix}{none}

\usepackage{framed}
\definecolor{shadecolor}{rgb}{0.94, 0.97, 1.0}

\newcommand{\markchange}[1]{\textcolor{black}{#1} }

\title{Energy-Based Learning for Cooperative Games,
with Applications to Valuation Problems in Machine Learning
}

\author{{Yatao Bian$^1$\thanks{Correspondence to: Yatao Bian <\texttt{yatao.bian@gmail.com}>}}, \; {Yu Rong}$^1$, {Tingyang Xu}$^1$, {Jiaxiang Wu}$^1$, {Andreas Krause}$^2$,  {Junzhou Huang}$^1$ \\
$^1$ Tencent AI Lab  \qquad \qquad \qquad $^2$ ETH Z{\"u}rich
}

\iclrfinalcopy

\begin{document}

\maketitle

\begin{abstract}%

Valuation problems, such as  feature interpretation, data valuation and model valuation for ensembles, become increasingly more important in many machine learning applications.  Such problems are commonly addressed via  well-known game-theoretic criteria,  such as the Shapley value or Banzhaf value.
In this work, we present a novel energy-based treatment for cooperative games, with a theoretical justification via the maximum entropy principle.
Surprisingly, through mean-field variational inference in the energy-based model, we recover classical game-theoretic valuation criteria
by conducting \emph{one-step} of fixed point iteration for maximizing the ELBO objective.  This observation also further supports existing criteria, as they can be seen as attempting to decouple the correlations among players.
By running the fixed point iteration for \emph{multiple} steps, we achieve a trajectory of the variational  valuations,  among which we define the valuation with the best conceivable decoupling error as the \emph{\varindex}. \markchange{
We prove that under uniform initialization,  these variational valuations all satisfy a set of game-theoretic
axioms.}
We empirically demonstrate that the proposed variational valuations \markchange{enjoy lower decoupling error and better valuation performance  on certain synthetic and real-world valuation problems.\footnote{Project page \& code: \url{https://valuationgame.github.io}}}

\end{abstract}

\section{Introduction}
\label{problem_description}

\begin{wrapfigure}{tr}{0.45\textwidth}
\vspace*{-5.8ex}
  \begin{center}
    \includegraphics[width=0.45\textwidth]{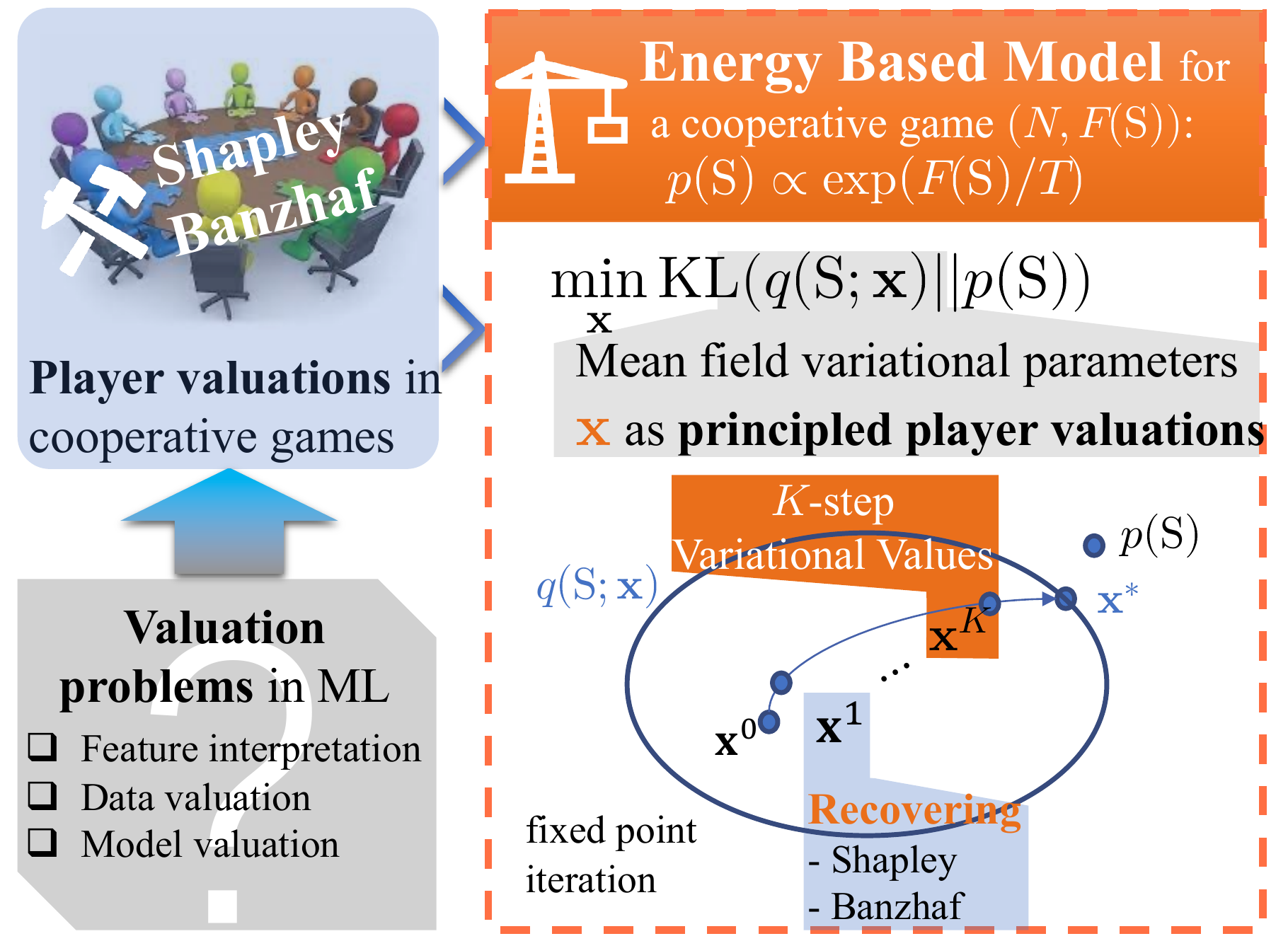}
  \end{center}
  \vspace*{-2ex}
  \caption{
  An energy based treatment of cooperative games leads to a series of new player valuations: $K$-step variational values,  satisfying basic valuation axioms: null player, marginalism \& symmetry.}
  \vspace*{-4.0ex}
  \label{fig_overview}
\end{wrapfigure}
\markchange{
Valuation problems are becoming increasingly more significant in various
machine learning applications, ranging from feature interpretation \citep{lundberg2017unified}, data valuation \citep{ghorbani2019data} to  model valuation for ensembles \citep{rozemberczki2021shapley}.}  They are often formulated as a player valuation  problem in cooperative games.  A cooperative game $(\groundset, F(S))$ consists of a grand coalition  $\groundset = \{1,..., n\}$ of $n$ players and a value function (a.k.a.~characteristic function) $F(S): 2^\groundset \rightarrow \R$ describing the collective payoff of a coalition/cooperation $S$.  A fundamental problem in cooperative  game theory is to assign an importance vector (i.e., solution concept) $\phi(F) \in \R^n$ to  $n$ players.

In this paper, we explore a {\em probabilistic treatment} of cooperative games $(\groundset, F(S))$.
\markchange{
Such a treatment makes it possible to conduct learning and inference
in a unified manner, and will yield connections with classical
valuation criteria. }
Concretely, we seek a probability distribution over coalitions $p(\BS=S)$\footnote{Note that  distributions over subsets of $\groundset$ are  equivalent to   distributions  of   $|\groundset|=n$    binary random  variables  $X_1, ..., X_n\in \{0,1\}$: We use $X_i$ as the indicator function of the event $i\in S$, or $X_i = [i\in S]$.
With slight abuse of notation, we use $\BS$ as a random variable represented as sets and often abbreviate $p(\BS=S)$ as $p(S)$.}, measuring the odds that a specific coalition $S$ happens.
Generally, we consider distributions where the probability of a coalition $p(S)$ grows monotonically with the payoff $F(S)$.

Among all the possible probability mass functions (pmfs), how should we construct the proper $p(S)$? We advocate to choose the pmf with the maximum entropy $\entropy{p}$.
This principle makes sense since maximizing the entropy minimizes the amount of prior information built into the distribution. In other words, it amounts to assuming nothing  about what is unknown, i.e.,  choosing the most ``uniform'' distribution.   Now finding a proper $p(S)$ becomes the following constrained optimization problem: suppose each coalition  $S$ is associated with a payoff $F(S)$ with probability $p(S)$. We would like to maximize the entropy $\entropy{p} = - \sum_{S\subseteq \groundset} p(S) \log p(S)$, subject to the constraints that $\sum_S p(S) = 1, p(S)\geq 0$ and $\sum_S p(S) F(S) = \mu$ (i.e., the average payoff is known as $\mu$).  Solving this optimization problem (derivation in   \cref{ap_maxent}), we reach the \emph{maximum entropy distribution}:
\begin{align}\label{eq_ebm}
    p(S) = \frac{\exp(F(S)/T)}{\parti}, \quad \parti:=\sum\nolimits_{S'\subseteq \groundset} \exp(F(S')/T),
\end{align}
where $T>0$ is the temperature.  This is an energy-based model  \citep[EBM, cf.][]{lecun2006tutorial} with $-F(S)$ as the energy function.

The above energy-based treatment admits two benefits:  i)  Where supervision is available, it enables learning of value functions $F(S)$ through efficient training techniques for energy-based learning, such as
noise contrastive estimation \citep{gutmann2010noise} and score matching \citep{hyvarinen2005estimation}. ii) Approximate inference techniques such as variational inference or sampling can be adopted to solve the valuation problem. \markchange{
Specifically, it enables to perform  mean-field variational inference where
parameters of the inferred surrogate distribution can be used as principled
player valuations.}

Below, we explore mean-field variational inference for the energy-based formulation (\cref{fig_overview}).
Perhaps surprisingly, by conducting only  \emph{one-step} fixed point iteration for maximizing the mean-field (ELBO) objective, we recover classical valuation criteria, such as the Shapley value \citep{shapley1953value} and the Banzhaf value \citep{penrose1946elementary,banzhaf1964weighted}.
This observation also further supports existing criteria, motivating them as decoupling the correlations among players via the mean-field approach.
By running the fixed point iteration for \emph{multiple} steps, we achieve a trajectory of valuations,  among which we define the valuation with the best conceivable decoupling error as the \emph{\varindex}.
Our major contributions can be summarized as below:

    i) We present a theoretically justified energy-based treatment for cooperative games.  Through mean field inference, we provide a unified perspective on popular game-theoretic criteria. %
    This provides an alternative motivation of existing criteria via a \emph{decoupling} perspective, i.e., decoupling correlations among $n$ players through the mean-field approach.
    ii) In pursuit of better decoupling performance, we propose to run fixed point iteration for \emph{multiple} steps, which generates a trajectory of valuations. \markchange{Under uniform initializations, they all satisfy a set of game-theoretic axioms, which are required for being suitable valuation criteria.   We  define  the  valuation  with  the best conceivable decoupling error as the {\varindex}.
    iii) Synthetic and real-world experiments demonstrate  intriguing properties of the proposed \varindex, including lower decoupling error and better valuation performance.}

\section{Preliminaries and Background}

\textbf{Notation.}
We assume
$\chara_i\in\R^n$ being  the standard $i^\text{th}$ basis vector and use boldface letters $\x\in \R^\groundset$ and $\x\in \R^n$
interchangebly to indicate an $n$-dimensional vector, where $x_i$ is
the $i^\text{th}$ entry of $\x$.
By default, $f(\cdot)$ is used to denote a continuous function, and
$F(\cdot)$ to represent a set function.
For a differentiable function $f(\cdot)$, $\nabla f(\cdot)$ denotes its gradient.
$\sete{x}{i}{k}$ is the operation of setting the
$i^\text{th}$ element of $\x$ to $k$, while keeping all other elements
unchanged, i.e., $\sete{x}{i}{k}=\x-x_i \bas_i + k\bas_i$.
For two sets $S$ and $T$,  $S+T$ and $S-T$ represent set union and set difference, respectively. $|S|$ is the cardinality of $S$. $i$ is used to denote the singleton $\{i\}$ with a bit abuse of notation.

\textbf{Existing valuation criteria.}
Various valuation criteria have been proposed from the area of cooperative games,
amongst them the most famous ones are the Shapley value \citep{shapley1953value} and the Banzhaf value, which is extended from
the Banzhaf power index \citep{penrose1946elementary,banzhaf1964weighted}.
For the Shapley value, the importance assigned to player $i$ is:
\begin{align}\label{def_banzhaf}
\shapley_i  = \sum\nolimits_{S\subseteq \groundset - i} [F(S + i) - F(S)] \frac{|S|! (n - |S| -1)!}{n!}.
\end{align}
One can see that it gives less weight to $n/2$-sized coalitions.
The Banzhaf value
assigns the following
importance to player $i$:
\begin{align}\label{def_shapley}
\banzhaf_i & = \sum\nolimits_{S\subseteq \groundset - i} [F(S + i) - F(S)] \frac{1}{2^{n - 1}},
\end{align}
which uses uniform weights for all the coalitions.
See \citet{greco2015structural} for a comparison of them.

\textbf{Valuation problems in machine learning.} Currently, most  classes of valuation problems \citep{lundberg2017unified,ghorbani2019data,sim2020collaborative,rozemberczki2021shapley} use Shapley value as the valuation criterion.  Along  with the rapid progress of model interpretation in the past decades \citep{zeiler2014visualizing,ribeiro2016should,lundberg2017unified,sundararajan2017axiomatic,petsiuk2018rise,wang2021shapley2}, attribution-based  interpretation aims to assign importance  to the features for a specific data instance $(\x\in \R^\groundset, y)$ given a black-box model $\M$. Here each feature maps to a player in the game $(\groundset, F(S) )$, and the value function $F(S)$ is usually the model response, such as the predicted probability for classification problems, when feeding a subset $S$ of features to $\M$.  The data valuation problem \citep{ghorbani2019data} tries to assign values to the samples in the training dataset $\groundset = \{(\x_i, y_i)\}_1^n$ for general supervised machine learning: one training sample corresponds to one player, and the value function $F(S)$ indicates the predictor performance on some test dataset given access to only a subset of the training samples in $S$.  Model valuation in ensembles \citep{rozemberczki2021shapley} measures importance of individual models in an ensemble in order to correctly label data points from a dataset, where each pre-trained model maps to a player and the value function measures the predictive performance of subsets of models.

\section{Related Work}

\textbf{Energy-based modeling.}
Energy based learning \citep{lecun2006tutorial} is a classical learning framework that  uses an energy function $E(\x)$ to measure
the quality of a data point $\x$.
Energy based models have  been applied to different domains,  such as data generation \citep{deng2020residual},  out-of-distribution detection \citep{liu2020energy},
reinforcement learning \citep{haarnoja2017reinforcement},  memory modeling \citep{bartunov2019meta},  discriminative learning \citep{grathwohl2019your,gustafsson2020train} and
biologically-plausible training \citep{scellier2017equilibrium}.
Energy based learning admits principled training methods, such as contrastive divergence \citep{hinton2002training},  noise contrastive estimation \citep{gutmann2010noise} and score matching \citep{hyvarinen2005estimation}.
For approximate inference, sampling based approaches are mainly MCMC-style algorithms, such as stochastic gradient Langevin dynamics \citep{Welling2011}.
For a wide class of EBMs with  submodular or supermodular energies \citep{djolonga14variational}, there exist provable mean field inference algorithms with constant factor approximation guarantees \citep{bian2019optimalmeanfield,sahin2020sets,bian2020continuous}.

\textbf{Shapley values in machine learning.}
Shapley values have been extensively used for valuation problems in machine learning,  including attribution-based interpretation \citep{lipovetsky2001analysis,cohen2007feature,strumbelj2010efficient,owen2014sobol,datta2016algorithmic,lundberg2017unified,chen2018shapley,lundberg2018consistent,kumar2020shapley,williamson2020efficient,covert2020understanding,wang2021shapley}, data  valuation \citep{ghorbani2019data,jia2019empirical,jia2019towards,wang2020principled,fan2021improving}, collaborative machine learning \citep{sim2020collaborative} and recently, model valuation in ensembles \citep{rozemberczki2021shapley}. For a  detailed overview of papers using Shapley values for feature interpretation, please see \cite{covert2020explaining} and the references therein.
To alleviate the exponential computational cost of exact evaluation, various methods have been proposed to approximate Shapley values in polynomial time \citep{ancona2017towards,ancona2019explaining}.
\cite{owen1972multilinear}  proposes the multilinear extension purely as a representation of cooperative games and \cite{okhrati2021multilinear} use it to develop sampling algorithms for Shapley values.

\section{Valuation for Cooperative Games: A Decoupling Perspective}
\label{sec_decoupling}

In the introduction, we have asserted that under the setting of cooperative games, the Boltzmann distribution (see \cref{eq_ebm}) achieves the maximum entropy  among all of the pmf functionals.
One can naturally view the importance assignment problem of cooperative games as a {\em decoupling problem}: The $n$ players in  a game $(\groundset, F(S))$  might be arbitrarily correlated  in a very complicated manner.  However,  in order to assign each of them an {\em individual} importance value,  we have to decouple their interactions, which can be viewed as a way to simplify their correlations.

\looseness -1 We therefore consider a surrogate  distribution $q(S; \x) $ governed by parameters in $\x$.   $q$ has to be simple, given our intention to decouple the correlations among the $n$ players.   A natural choice is to restrain $q(S; \x) $ to be fully factorizable, which leads to a mean-field approximation of $p(S)$.  The simplest form of  $q(S; \x)$ would be a $n$ independent Bernoulli distribution, i.e.,
$q(S; {\x}):= \prod_{i\in S}x_i \prod_{j\notin S}(1-x_j), \x\in
[0,1]^n$.
Given a  divergence measure $D(\cdot \| \cdot)$ for probability distributions, we can define the \emph{best conceivable decoupling error} to be the
divergence between $p$ and the best possible $q$.

\begin{definition}[Best Conceivable Decoupling Error]
\label{def_decoupling_error}
Considering a cooperative game  $(\groundset,  F(S))$, and given a divergence measure $D(\cdot \| \cdot)$ for probability distributions,
the \emph{decoupling error} is defined as the divergence between $q$ and $p$: $D(q \| p)$, and
the best conceivable decoupling error is defined as the
divergence between the best possible $q$ and $p$:
\begin{align}
  D^* :=  \min_{q}  D(q \| p).
\end{align}
\end{definition}
Note that the \emph{best conceivable decoupling error}
$D^*$ is closely related to the intrinsic coupling amongst $n$ players: if all the players are already independent with each other, then $D^*$ could be zero.

\subsection{Mean Field Objective for EBMs}
\label{sec_meanfield_lowerbounds}

\looseness -1 If we consider the \emph{decoupling error} $D(q \| p)$ to be the Kullback-Leibler divergence between $q$ and $p$, then we recover the mean field approach\footnote{Notably, one could also apply the reverse KL divergence $\kl{p}{q}$, which would lead to an expectation propagation \citep{minka2001expectation} treatment of cooperative games.}.
Given the EBM formulation in \cref{eq_ebm},
the classical mean-field inference approach aims to approximate $p(S)$ by
a fully factorized product  distribution
$q(S; {\x}):= \prod_{i\in S}x_i \prod_{j\notin S}(1-x_j), \x\in
[0,1]^n$,
by minimizing the distance measured w.r.t.~the Kullback-Leibler
divergence between $q$ and $p$.
Since $\kl{q}{p}$ is non-negative, we have:
\begin{align}\notag
& 0\leq \kl{q}{p} =  \sum\nolimits_{S\subseteq \groundset} q(S; {\x})
\log\frac{q(S; {\x})}{p(S)}\\
& = -  \E_{q(S; {\x})} [\log p(S)]  -  \entropy{q(S; {\x})}  \\\label{eq_ebm_logp}
& = -  \E_{q(S; {\x})} [ \frac{F(S)}{T} - \log \parti ]  -  \entropy{q(S; {\x})}  \\
&=
  -\sum_{S\subseteq \groundset} \frac{F(S)}{T}  \prod_{i\in S}x_i \prod_{j\notin S}(1-x_j)+ \sum\nolimits_{i=1}^{n} [x_i\log x_i + (1-x_i)\log(1-x_i)] + \log \parti.
\end{align}
In \cref{eq_ebm_logp} we plug in the EBM formulation that  $\log p(S) = \frac{F(S)}{T} - \log \parti$.
 Then  one can get
\begin{align}\notag
\log \parti  & \geq \sum_{S\subseteq \groundset} \frac{F(S)}{T} \prod_{i\in S}x_i \prod_{j\notin S}(1-x_j)- \sum\nolimits_{i=1}^{n} [x_i\log x_i + (1-x_i)\log(1-x_i)] \\ \label{elbo}
& = \frac{\multi(\x)}{T}   +  \entropy{q(S; \x)}   := (\text{\textcolor{blue}{ELBO}})
\end{align}
where $\entropy{\cdot}$ is the entropy, ELBO stands for the {\em evidence lower bound},  and
\begin{align}
    \multi(\x) := \sum\nolimits_{S\subseteq \groundset}  F(S) \prod\nolimits_{i\in S}x_i \prod\nolimits_{j\notin S}(1-x_j),  \x \in [0, 1]^n,
\end{align}
is the {\em multilinear extension} of  $F(S)$  \citep{owen1972multilinear,calinescu2007maximizing}.
Note that the multilinear extension plays a central role in modern combinatorial optimizaiton techniques \citep{feige2011maximizing}, especially for guaranteed submodular maximization problems \citep{krause2014submodular}.

Maximizing $(\text{ELBO})$ in \cref{elbo} amounts to
minimizing the Kullback-Leibler divergence between $q$ and $p$. If one solves this optimization problem to optimality, one can obtain the $q(S; \x^*)$ with the best conceivable decoupling error.  Here $x^*_i$ describes the odds that player $i$ shall participate in the game, so it can be naturally used to define the importance score of each player.

\begin{definition}[\varindex of Cooperative Games]
Consider a cooperative game  $(\groundset,  F(S))$ and its mean field approximation. Let  $\x^*$ be the variational marginals with the best conceivable decoupling error, we define $\s^*:= T\sigma^{-1}(\x^*)$ to be the variational index of the game. Formally,
\begin{align}\label{eq_kl}
    \x^* = \argmin_{\x} \kl{q(S; \x)}{p(S)},
\end{align}
where $\x^*$ can be obtained by maximizing the ELBO objective in \cref{elbo}, and $\sigma^{-1}( \cdot )$ is the inverse of the sigmoid function, i.e. $\sigma^{-1}(x) = \log\frac{x}{1-x}$. For a vector it is applied element-wise.
\end{definition}

\subsection{Algorithms for Calculating the \varindex}
\label{subsec_algorithms}

\textbf{Equilibrium condition.}
For coordinate $i$, the partial
derivative of the multilinear extension is $\nabla_i\multi(\x)$, and for
the entropy term, it is $\nabla_i \entropy{q(S; \x)} = \log \frac{1-x_i}{x_i}$. By setting the partial derivative of ELBO in \cref{elbo} to be 0, we have the equilibrium condition:
\begin{align}
x^*_i = \sigma(\nabla_i {{f^F_{\text{mt}}}}(\mathbf{x^*})/T) = \bigl(1+ \exp(-
    \nabla_i {{f^F_{\text{mt}}}}(\mathbf{x^*})/T \bigr)^{-1}, \quad \forall i \in N,
\end{align}
where $\sigma$ is the sigmoid function. This equilibrium condition implies that one cannot change  the value assigned to any player in order to further improve the overall   decoupling performance.
It also implies the fixed point iteration
$x_i \leftarrow \sigma(\nabla_i \multi(\x)/T)$.
    When  updating each coordinate sequentially, we recover the
classic naive mean field algorithm as shown in
\cref{app_alg}.

Instead, here we suggest to use the full-gradient  method shown in \cref{alg_mfi_ga} for maximizing the ELBO objective. As we will see later, the resultant valuations satisfy certain game-theoretic axioms.
It needs an initial marginal vector $\x^{0}\in [0,1]^n$ and the number of epochs $K$.  After $K$ steps of fixed point iteration, it returns the estimated marginal $\x^{K}$.
\vspace{-.1cm}
\begin{algorithm}[htbp]
	\caption{{Mean Field Inference with Full Gradient}: \textcolor{blue}{$\mfi(\x; K)$}}\label{alg_mfi_ga}
	\KwIn{A cooperative game $(\groundset,  F(S))$ with $n$ players. Initial marginals  $\x^\pare{0}\leftarrow \x \in [0, 1]^n$.   \#epochs $K$.
	}
	\KwOut{Marginals after $K$ steps of iteration: $\x^\pare{K}$}
	\For{$k = 1 \rightarrow K$}{
		{$\x^\pare{k} \leftarrow \sigma(\nabla \multi(\x^\pare{k-1})/T) = \bigl(1+ \exp(-
    \nabla \multi(\x^\pare{k-1})/T \bigr)^{-1}$ \;}
	}
\end{algorithm}
\vspace{-.1cm}

In case \cref{alg_mfi_ga} solves the optimization problem to optimality, we obtain the \varindex. However,  maximizing ELBO is in general a non-convex/non-concave problem, and hence one can only ensure reaching a stationary solution. Below, when we say \varindex, we therefore refer to its approximation obtained via \cref{alg_mfi_ga} by default.
Meanwhile, the $\mfi(\x; K)$ subroutine also defines a series of marginals, which enjoy interesting properties as we show in the next part.
So we define variational valuations through intermediate solutions of $\mfi(\x; K)$.
\vspace{-.1cm}
\begin{snugshade}
\vspace{-.2cm}
\begin{definition}[$K$-Step Variational Values]
\label{def_kstep_var_vals}
Considering a cooperative game  $(\groundset,  F(S))$ and its mean field approximation by \cref{alg_mfi_ga}, we define the \emph{$K$-Step Variational Values} initialized at $\x$ as:
\begin{align}
 T \sigma^{-1}(\mfi(\x; K)),
\end{align}
where $\sigma^{-1}()$ is the inverse of the sigmoid function ($\sigma^{-1}(x) = \log\frac{x}{1-x}$).
\end{definition}
\vspace{-.2cm}
\end{snugshade}
\vspace{-.2cm}
Notice  when running  more steps,  the $K$-Step variational value will be more close to the \varindex.
The gradient $\nabla \multi(\x)$ itself is defined with respect to an exponential sum via the multilinear extension.
Next we show how it can be approximated via principled sampling methods.

\textbf{Sampling methods for estimating the partial derivative.}
The partial derivative follows,
\begin{flalign}\label{eq_partial_derivative}
\nabla_i \multi(\x)  &
= \E_{q(S ;  (\sete{x}{i}{1}))} [F(S)]  - \E_{q(S;  (\sete{x}{i}{0}))} [F(S)]\\\notag
& =\multi(\sete{x}{i}{1}) - \multi(\sete{x}{i}{0})
\\\notag
&  =  \sum_{S\subseteq \groundset, S\ni i } F(S)
\prod_{j \in S - i}x_j
\prod_{j'\notin S}(1-x_{j'})
 \quad  - \sum_{S\subseteq
	\groundset - i }\ F(S) \prod_{j\in
	S} x_j \prod_{j'\notin S, j'\neq i}
(1-x_{j'})\\\notag
& = \sum_{S\subseteq
	\groundset - i  }\ [F(S+ i) - F(S)]   \prod_{j\in
	S} x_j \prod_{j'\in \groundset -  S - i}
(1-x_{j'})\\\notag
& =\! \sum_{S\subseteq
	\groundset - i  }\ [F(S+ i) - F(S)]   q(S; (\sete{x}{i}{0})) = \E_{S \sim q(S; (\sete{x}{i}{0}))}\ [F(S+ i) - F(S)].
\end{flalign}
All of the variational criteria are based on the calculation of the partial derivative   $\nabla_i \multi(\x) $,  which can be approximated by Monte Carlo sampling since  $\nabla_i \multi(\x)  = \E_{S \sim q(S; (\sete{x}{i}{0}))}\ [F(S+ i) - F(S)]$:  we first sample $m$ coalitions $S_k, k=1,...,m$ from the surrogate distribution $q(S; (\sete{x}{i}{0}))$, then approximate the expectation by the average $\frac{1}{m}\sum_{k=1}^m\ [F(S_k+ i) - F(S_k)]$.
According to the Chernoff-Hoeffding bound \citep{hoeffding1963probability}, the approximation will be arbitrarily close to the true value with increasingly more samples: With probability at least
$1- \exp(-m\epsilon^2/2)$, it holds that
$|\frac{1}{m}\sum_{k=1}^{m} [F(S_k+ i) - F(S_k)]  - \nabla_i \multi(\x)| \leq \epsilon \max_S
|F(S+i) - F(S)| $, for all $\epsilon > 0$.

\textbf{Roles of the initializer $\x^\pare{0}$ and the temperature $T$.}
This can be understood in the following respects: 1) The initializer $\mathbf{x}^0$ represents the initial credit assignments to the $n$ players, so it denotes the prior knowledge/initial belief of the contributions of the players;
2) If one just runs \cref{alg_mfi_ga} for one step,  $\mathbf{x}^0$ matters greatly to the output. However, if one runs \cref{alg_mfi_ga} for many steps,
$\mathbf{x}^k$ will converge to the stationary points of the ELBO objective.
Empirically, it takes around 5$\sim$10 steps to converge.  %
The temperature  $T$ controls the ``spreading'' of importance  assigned to the players: A higher $T$ leads to flatter assignments,  and a lower $T$ leads to more concentrated assignments.

\textbf{Computational efficiency of calculating \varindex and variational values.}
\cref{alg_mfi_ga} calculates the $K$-step variational values, \markchange{1-step variational value has the same computational cost as that of
Banzhaf value and of the integrand of the line integration of Shapley value in \cref{eq_shapley_line} below, since they all need to evaluate $\nabla \multi(\x) $.}
Sampling methods could help with approximating all of the three criteria when there are a large number of players.
The Variational Index can be approximated by the $K$-step variational value, where the number $K$ depends on when \cref{alg_mfi_ga} converges.  One can easily show that, under the setting of maximizing ELBO,  $\mathbf{x}^k$ will converge to some stationary point $\mathbf{x}^*$, based on the analysis of mean field approximation in \cite{wainwright2008graphical}.
We have also empirically verified the convergence rate of \cref{alg_mfi_ga} in \cref{subsec_convergence},
and find that it  converges within  5 to 10 steps.
So the  computational cost is roughly similar as that of Shapley value and Banzhaf value.

\subsection{Recovering Classical Criteria}

\looseness -1 Perhaps surprisingly, it is possible to recover classical valuation criteria via the $K$-step variational values as
in \cref{def_kstep_var_vals}.
Firstly,  for Banzhaf value,  by comparing with \cref{eq_partial_derivative} it reads,
\begin{align}
\banzhaf_i  = \sum\nolimits_{S\subseteq \groundset - i} [F(S + i) - F(S)] \frac{1}{2^{n - 1}}
& = \nabla_i \multi(0.5*\mathbf{1})  =  T \sigma^{-1}(\mfi(0.5*\mathbf{1}; 1) ),
\end{align}
which is the $1$-step variational value initialied at $0.5*\mathbf{1}$.
We can also recover the Shapley value through its connection to the multilinear extension \citep{owen1972multilinear,grabisch2000equivalent}:
\begin{align}\label{eq_shapley_line}
\shapley_i  = \int_0^1 \nabla_i \multi(x\mathbf{1}) dx  =  \int_0^1 T \sigma^{-1}(\mfi(x\mathbf{1}; 1) ) dx,
\end{align}
where the integration denotes integrating the partial-derivative of the multilinear extension
along the main diagonal of the unit hypercube.
A self-contained proof is given in \cref{app_recoving_proof}.

These insights offer a novel, unified interpretation of the two classical valuation indices:  both the Shapley value and Banzhaf value can be viewed as approximating the variational index by running {\em one step} of fixed point iteration for the decoupling (ELBO) objective.  Specifically, for the Banzhaf value,  it initializes $\x$ at $0.5*\mathbf{1}$, and runs one step of fixed point iteration.  For the Shapley value,  it also performs a one-step fixed point approximation. However, instead of starting at a single initial point, it averages over all possible initializations through the line integration in \cref{eq_shapley_line}.

\textbf{Relation to probabilistic values.}
Probabilistic values  for games \citep{weber1988probabilistic,monderer2002variations} capture a class of solution concepts, where the value of each player is given by some averaging of the player's marginal contributions to coalitions, and the weights depend on the coalitions only.
According to  \cite[Equation (3.1)]{monderer2002variations}, a solution $\phi$ is called a probabilistic value, if for each player $i$, there exists a probability $p^i \in \Delta(C^i)$, such that $\phi_i$ is the expected marginal contribution of $i$ w.r.t. $p^i$. Namely,
$
\phi_i = \sum_{S\in C^i} p^i(S) [F(S+i) - F(S)],
 $
where $C^i$ is the set of all subsets of $N-i$, and $\Delta(C^i)$ is the set of all probability measures on $C^i$.
One can easily see that, for any fixed $\mathbf{x}$, 1-step variational value in \cref{def_kstep_var_vals}  is a probabilistic value with
$p^i(S) = q(S;  (\mathbf{x}|x_i \leftarrow 0))$,
where $q(S;  \mathbf{x})$ is the surrogate distribution in our EBM framework.

\subsection{Axiomatisation of $K$-Step Variational Values}

Our EBM framework introduces a series of variational values controlled by $T$ and the running step number $K$.
We now establish that the variational values $T \sigma^{-1} (\mfi(\x; K))$ in \cref{def_kstep_var_vals} satisfy certain game-theoretic axioms (see \cref{appen_axioms} for definitions of five common axioms: Null player, Symmetry, Marginalism, Additivity and Efficiency).
\looseness -1 We prove that  all the variational values in the trajectory satisfy three fundamental axioms: null player, marginalism and symmetry.   The detailed  proof is deferred to \cref{app_proof_thm1}.
We expect it to be very difficult to find \emph{equivalent} axiomatisations  of the series of variational values, which we leave for future work.
Meanwhile, our methods incur a decoupling and fairness tradeoff by tuning the hyperparameters $K$ and $T$.
\vspace{-.1cm}
\begin{snugshade}
\vspace{-.2cm}
\begin{restatable}[Axiomatisation of $K$-Step Variational Values of \cref{def_kstep_var_vals}]{theorem}{restattheoremone}
\label{thm_axiom}
If initialized uniformly, i.e.,  $\x^0 = x\textbf{1}, x\in [0, 1]$, all the  variational values in the  trajectory  $T\sigma^{-1}
(\mfi(\x; k)), k=1,2,3...$
satisfy  the  null player, marginalism and symmetry axioms.
\end{restatable}
\vspace{-.2cm}
\end{snugshade}
\vspace{-.2cm}
\markchange{
According to \cref{thm_axiom},   our proposed $K$-step variational values satisfy the minimal  set of axioms often associated with appropriate valuation criteria.
Note that specific realizations of the $K$-step variational values can also satisfy more axioms, for example,
the $1$-step variational value initialized at $0.5*\mathbf{1}$
 also satisfies the additivity axiom. Furthermore, we have the following observations:}

\markchange{
\textbf{Satisfying more axioms is not essential for valuation problems.} Notably, in cooperative game theory, one line of work is to seek for solution concepts that would satisfy more axioms.
However, for valuation problems in machine learning, this is arguably not essential.  For example, similar as argued by \cite{ridaoui2018axiomatisation},  efficiency does not make sense for certain games.
We give  a simple illustration  in \cref{append_misc}, which further shows that whether more axioms shall be considered  really depends on the specific scenario being modeled, which will be left for important future work. }

\section{Empirical Studies}
Throughout the experiments, we are trying to understand the following: 1) Would the proposed \varindex have lower decoupling error compared to others? 2) Could the proposed \varindex gain benefits compared to the classical valuation criteria for valuation problems?

\begin{figure}[tbp]
	\centering
		\vspace{-1cm}
	\includegraphics[width=1\textwidth]{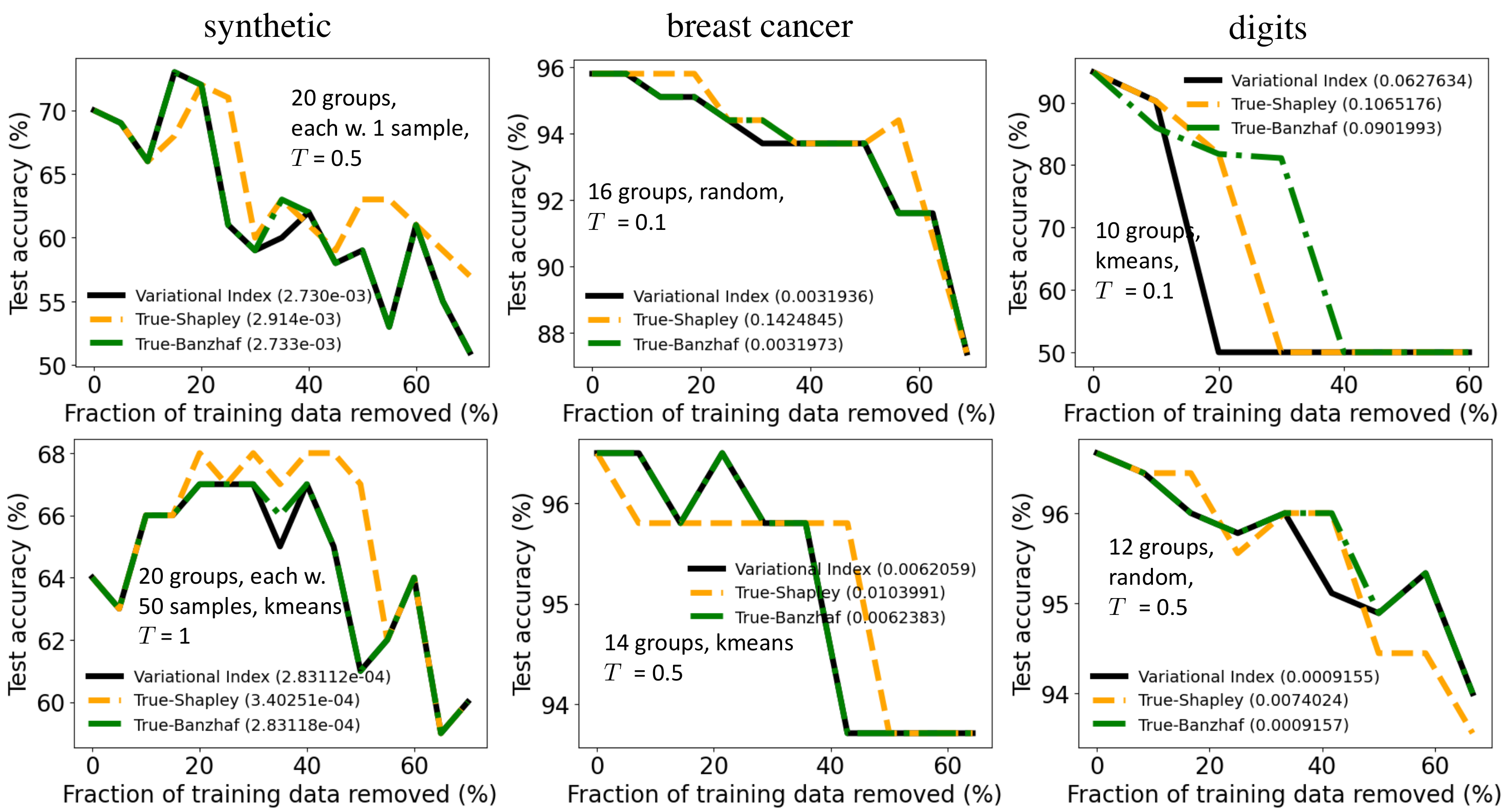}
	\caption{Data removal results. Numbers in the legend are the \emph{decoupling} errors. Columns: 1st: synthetic data; 2nd: breast cancer data with 569 samples; 3rd: digits data with 1797 samples. Specific configurations (e.g., temperature) are put inside the figure texts.}
	\label{fig_data_removing}
		\vspace{-.4cm}
\end{figure}

\looseness -1 Since we are mainly  comparing the quality of different criteria, it is necessary  to rule out the influence of approximation errors when estimating their values.  So we focus on small-sized problems where one can compute the exact values of these criteria in a reasonable time.  Usually this requires the number of players to be no more than 25. Meanwhile, we have also conducted experiments with a larger number of players in \cref{apendix_large_players}, in order to show the efficiency of sampling methods.
We choose $T$ empirically from the values of 0.1, 0.2, 0.5, 1.0.
We choose $K$  such that \cref{alg_mfi_ga} would converge. Usually, it takes around 5 to 10 steps to converge. We give all players a fair start, so $\x^0$
 was initialized to be $0.5 \times \mathbf{1}$. Code is available at \url{https://valuationgame.github.io}.

We first conduct synthetic experiments on submodular games (details defered to \cref{exp_syn}), in order to verify the quality of  solutions in terms of the true marginals $p(i\in \mathbf{S})$.   One can conclude that \varindex obtains better performance in terms of MSE and Spearman's rank correlation compared to
the one-point solutions (Shapley value and Banzhaf value) in all experiments.
More experimental results on data point and feature valuations are deferred to \cref{appd_more_exps}.

\subsection{Experiments on Data Valuations}

We follow the setting of \cite{ghorbani2019data} and reuse the code of \url{https://github.com/amiratag/DataShapley}.
We conduct data removal: training samples are sorted according to the valuations returned by different criteria, and then samples are removed in that order to check how much the test accuracy drops.  Intuitively,  the best criteria would induce the fastest drop of performance.
We experiment with the following  datasets:
a) Synthetic datasets similar as that of \cite{ghorbani2019data}; b)
The breast cancer dataset, which is a binary classification dataset with 569 samples;  c)  The digits dataset, that is a 10-class classification dataset with 1797 samples. The above two datasets are both from  UCI Machine Learning repository (\url{https://archive.ics.uci.edu/ml/index.php}).
Specifically, we cluster data points into groups and  studied two  settings: 1) Grouping the samples randomly; 2) Clustering the samples with the k-means algorithm.  For simplicity, we always use equal group sizes.  The data point removal corresponds to singleton groups.
\cref{fig_data_removing} shows the results.
One can observe that in certain situations the \varindex achieves the fastest drop rate. It always achieves the lowest decoupling error (as shown in the legends in each of the figures).
Sometimes  \varindex and Banzhaf show similar performance. We expect that this is because the Banzhaf value is a one-step approximation of \varindex, and for the specific problem considered, the ranking of the solutions does not change after one-step of fixed point iteration.
There are also situations where
the rankings of the three criteria are not very  distinguishable,
however, the specific values are also very different since the
decoupling error differs.

\subsection{Experiments on Feature Valuations/Attributions}

\begin{figure}[tbp]
	\centering
		\vspace{-1.2cm}
	\includegraphics[width=1\textwidth]{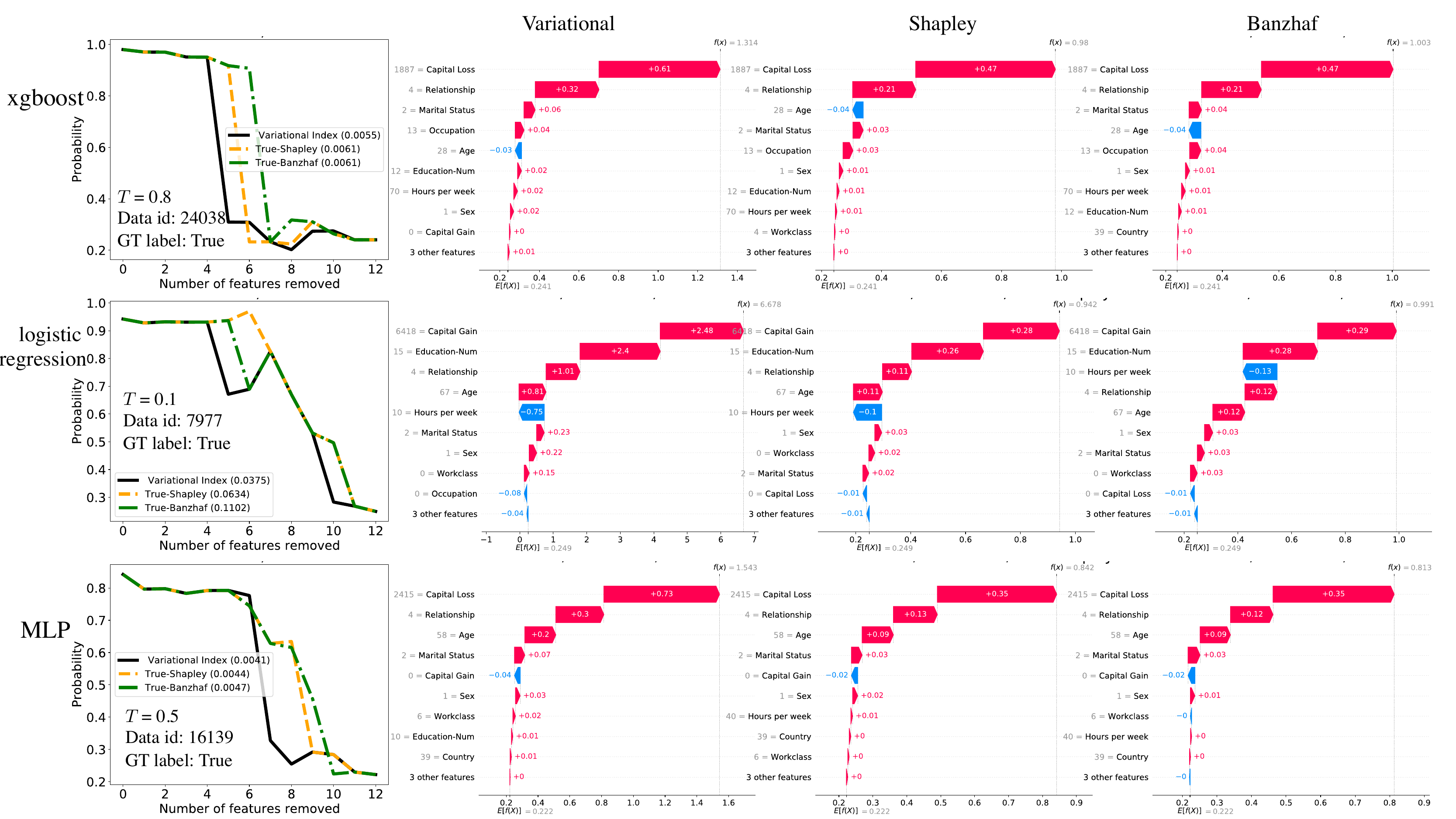}
	\caption{First column: Change of predicted probabilities when removing features. The \emph{decoupling error} is included in the legend.   Last three columns: waterfall plots of  feature importance.
	}
	\label{fig_feature_removing}
		\vspace{-.3cm}
\end{figure}
We follow the setting of \cite{lundberg2017unified} and reuse the code of \url{https://github.com/slundberg/shap} with an MIT License.
We train  classifiers  on the Adult dataset\footnote{\url{https://archive.ics.uci.edu/ml/datasets/adult}}, which predicts whether an adult's income exceeds 50k dollar per year based on census data.  It has
48,842 instances and  14 features  such as age, workclass, occupation, sex and capital gain (12 of them  used).

\textbf{Feature removal results.}
This experiment follows a similar fashion as the data removal experiment: we remove the features one by one according to the order defined by the returned criterion, then observe the change of predicted probabilities.  \cref{fig_feature_removing} reports the behavior of the three criteria.
The first row shows the results from an xgboost classifier (accuracy: 0.893), second row  a logistic regression classifier (accuracy: 0.842), third row  a multi-layer perceptron (accuracy: 0.861). For the probability dropping results, \varindex usually induces the fastest drop,  and it  always enjoys the smallest decoupling error, as expected from its mean-field nature.  From the waterfall plots, one can see that the three criteria indeed produce different rankings of the features. Take the first row for example. All criteria put ``Capital Loss'' and ``Relationship'' as the first two features.  However, the remaining features have different ranking:  \varindex and  Banzhaf  indicate that ``Marital Status'' should be ranked third, while Shapley ranks it in the fourth position.  It is hard to tell which ranking is the best because:  1)  There is no golden standard to determine the true ranking of features; 2) Even if there exists a ground truth ranking of some ``perfect model'', the trained xgboost model here might not be able to reproduce it, since it might not be aligned with the ``perfect model''.

\textbf{Average results.} We further provide the bar plots and averaged ranking across the adult datasets  in \cref{fig_stats}.   From the bar plots one can see that different criterion has slightly different values for each feature on average. Average rankings in the table demonstrate the difference: The three methods do not agree on the colored features, for example, ``Age'',  ``Education-Num'' and ``Captical Loss''.

\begin{figure}[tbp]
	\centering
		\vspace{-1.1cm}
	\includegraphics[width=1.02\textwidth]{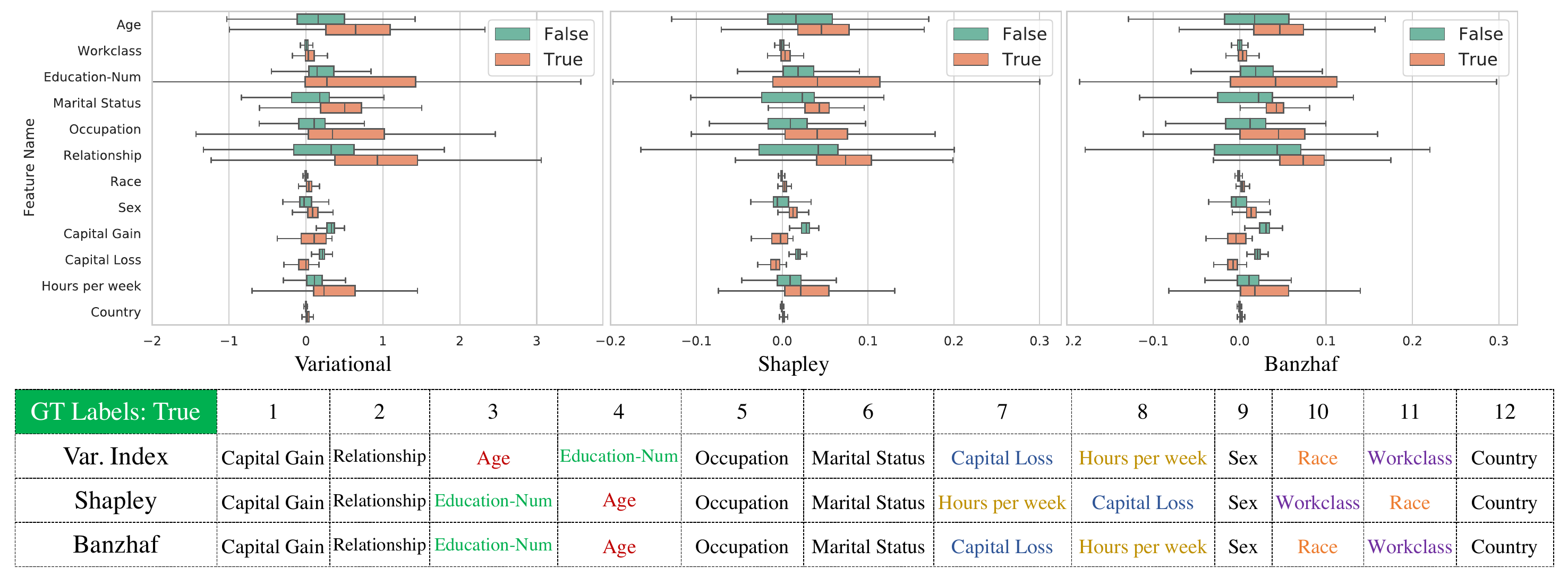}
	\caption{Statistics on valuations with the \emph{xgboost} classifier. First row: box plot of valuations.
	We always consider the predicted probability of the ground truth label. ``True'' means the samples with positive ground truth  label and ``False'' means with the negative ground truth label. Second row: Average ranking of the 12 features. Colored texts denote different rankings among the three criteria.}
	\label{fig_stats}
		\vspace{-.3cm}
\end{figure}

\subsection{Empirical Convergence Results of  \cref{alg_mfi_ga}}
\label{subsec_convergence}

\cref{tab_convergence} shows convergence results of \cref{alg_mfi_ga} on
feature and data valuation experiments. The value in the cells are the stepwise difference of $\mathbf{x}^k$ ,  $\frac{\|\mathbf{x}^k - \mathbf{x}^{k-1}\|^2}{n}$, which is a classical criterion to measure the convergence of iterative algorithms.
One can clearly see that
\cref{alg_mfi_ga} converges in 5 to 10 iterations.

\begin{table}[h!]
\vspace{-.3cm}
 \caption{Stepwise difference $\frac{\|\mathbf{x}^k - \mathbf{x}^{k-1}\|^2}{n}$ of \cref{alg_mfi_ga} for different experiments.}
\label{tab_convergence}
\centering
\footnotesize
\vspace{-.2cm}
 \begin{tabular}{r c c c c c c }
 \hline
   Step/Iteration Num    & 1 &	2 &	3 &	5	 & 9 & 10  \\
 \hline
Data Val (breast cancer) & 0.0023 & 3.61e-6 &1.53e-7 &2.77e-10 &9.12e-16&0 \\
Data Val (digits)&	0.00099&	5.93e-7&	1.46e-8&	8.92e-12&	9.25e-18	&0 \\
Data Val (synthetic)&	0.00059&	2.49e-8&	3.13e-10&	6.06e-14&	0&	0 \\
Feature Val (xgboost)&	0.0066&	1.68e-5&	8.71e-7&	2.35e-9&	1.75e-14&	9.25e-16 \\
Feature Val (LR)&	0.0092&	2.63e-5&	1.44e-6&	4.31e-9&	2.14e-15&	1.28e-16 \\
Feature Val (MLP)&	0.0040&	4.86e-6&	1.86e-7&	2.84e-10&	6.82e-16&	3.20e-17\\
 \hline
 \end{tabular}
\end{table}

\textsc{\large Discussions and Future Work.}
We have presented an energy-based treatment of cooperative games, in order to improve  the valuation  problem.
It is very worthwhile to explore more in the following directions:
1) Choosing the temperature $T$.
The temperature controls the level of  fairness since, when $T\rightarrow \infty$, all players have equal importance, when $T\rightarrow 0$, whereas a player has either 0 or 1 importance (assuming no ties).  Perhaps one can use an annealing-style algorithm in order to control the fairness level: starting with a high temperature and gradually decreasing it, one can obtain a series of importance values under different fairness levels.
2) Given the probabilistic treatment of cooperative games, one can naturally add priors over the players, in order to encode more domain knowledge.
It may also make sense to consider conditioning and marginalization in light of practical applications.
3) It is very interesting to explore the interaction of a group of players in the energy-based framework,  which would result in an ``interactive'' index among size-$k$ coalitions.
\section*{Ethics Statement and Broader Impact}

Besides the valuation problems explored in this work, cooperative game theory has already been applied to a wide range of disciplines, to name a few, economics, political science, sociology, biology, so this work could potentially  contribute to  broader domains as well.

\looseness -1 Meanwhile, we have to be aware of possible negative societal impacts, including:  1) negative side effects of the technology itself, for example,  possible unemployment issues due to the reduced amount of the need of valuations by human beings; 2)  applications in negative downstream tasks, for instance, the data point valuation technique could make it easier to conduct underground transactions of private data.

\section*{Reproducibility Statement}

All the datasets are publicly available as described in the main text.
In order to ensure reproducibility, we have made the  efforts in the following respects:  1) Provide code as  supplementary material. 2) Provide self-contained proofs of the main claims in \cref{app_recoving_proof,app_proof_thm1}; 3) Provide more  details on experimental configurations  and experimental results in \cref{appd_more_exps}.

{
\typeout{}
\bibliography{bib}
}

\newpage
\appendix
\appendixtitle{Appendix of ``Energy-Based Learning for Cooperative Games, with Applications to Valuation Problems in Machine Learning''}

\etocdepthtag.toc{mtappendix}
\etocsettagdepth{mtchapter}{none}
\etocsettagdepth{mtappendix}{subsection}
\tableofcontents

\section{Derivations of the Maximum Entropy Distribution}
\label{ap_maxent}

One may wonder why do we need energy-based treatment of valuation problems in machine learning?  Specifically, under the setting of cooperative games $(\groundset, F(S))$,  why we have to take the exponential form of EBM $p(S) \propto \exp(F(S)/T)$?  Because one may also formulate it as something else, say,  $p(S) \propto (1 +  |F(S)|)$?

\markchange{
In one word, because EBM is the maximum entropy distribution.
Being a maximum entropy distribution means minimizing the amount of prior information built into the distribution.
Another lens to understand it is that:
Since the distribution with the maximum entropy is the one that makes the fewest assumptions about the true distribution of data, the principle of maximum entropy can be seen as an application of Occam's razor.}
Meanwhile, many physical systems tend to move towards maximal entropy configurations over time \citep{jaynes1957information,jaynes1957information2}.

In the following we will give a derivation to show that $p(S) \propto \exp(F(S)/T)$ is indeed the maximum entropy distribution  for a cooperative game.
The derivation bellow is closely  following \cite{jaynes1957information,jaynes1957information2} for statistical mechanics.
 Suppose each coalition  $S$ is associated with a payoff $F(S)$ with probability $p(S)$. We would like to maximize the entropy $\entropy{p} = - \sum_{S\subseteq \groundset} p(S) \log p(S)$, subject to the constraints that $\sum_S p(S) = 1, p(S)\geq 0$ and $\sum_S p(S) F(S) = \mu$ (i.e., the average payoff is known as $\mu$).

Writing down the Lagrangian

\begin{align}
& L(p, \lambda_0, \lambda_1) :=  - \left[
\sum_{S\subseteq \groundset} p(S) \log p(S) +  \lambda_0 (\sum_{S\subseteq \groundset} p(S) - 1)
+ \lambda_1 (\mu - \sum_{S\subseteq \groundset} p(S) F(S))
\right]
\end{align}

Setting
\begin{align}
\fracpartial{L(p, \lambda_0, \lambda_1)}{p(S)}  =  - \left[
\log p(S) + 1 + \lambda_0 -  \lambda_1 F(S) \right]  \\
= 0
\end{align}
we get:
\begin{align}
    p(S) = \exp[-(\lambda_0 + 1 - \lambda_1 F(S) )]
\end{align}

\begin{align}
    \lambda_0 + 1 = \log \sum_{S\subseteq \groundset}  \exp (\lambda_1 F(S)) =: \log \parti
\end{align}
is the log-partition function. So,
\begin{align}
    p(S) = \frac{\exp[\lambda_1 F(S) )]}{\parti}
\end{align}

Note that the maximum value of the entropy is
\begin{align}
 & H_{\text{max}} =  - \sum_{S\subseteq \groundset} p(S) \log p(S)  \\
 & =   \lambda_0 + 1 - \lambda_1 \sum_{S\subseteq \groundset} p(S) F(S)  \\
 & = \lambda_0 + 1 - \lambda_1 \mu
\end{align}

So one can get,
\begin{align}
    \lambda_1 = - \fracpartial{H_{\text{max}}}{\mu} =: \frac{1}{T}
\end{align}
which defines the inverse temperature.

So we reach the exponential form of $p(S)$ as:
\begin{align}
    p(S) = \frac{\exp[F(S)/T)]}{\sum_{S\subseteq \groundset}  \exp [F(S)/T]}.
\end{align}

\section{Common Axioms of Valuation Criteria}
\label{appen_axioms}

Following the definitions in \cite{covert2020explaining}, the five common axioms are listed as bellow.  $\phi_i(F)$ denotes the value assigned to player $i$ in the game $(\groundset, F(S))$. Note that the notions might be slightly different in classical literature of game theory.

\textbf{Null player:}  For a player $i$ in the cooperative game $(\groundset,  F(S))$, if $F(S+i) = F(S)$ holds for all $S\subseteq \groundset - i$, then its value should be  $\phi_i(F) = 0$.

\textbf{Symmetry:}   For any two players $i, j$ in the cooperative game $(\groundset,  F(S))$, if $F(S+i) = F(S+j)$ holds for all $S\subseteq \groundset - i-j$, then it holds that   $\phi_i(F) = \phi_{j}(F)$.

\textbf{Marginalism:}   For two games $(\groundset,  F(S))$ and $(\groundset,  G(S))$ where all players have identical marginal contributions,  the players obtain equal valuations:   $F(S+i)  - F(S) = G(S+i)  - G(S)$ holds for all $(i, S)$, then it holds that   $\phi_i(F) = \phi_{i}(G)$.

\textbf{Additivity:}  For two games $(\groundset,  F(S))$ and $(\groundset,  G(S))$,
if they are combined, the total contribution of a player is equal to the sum of its individual contributions on each game: $\phi_i(F+G) =\phi_{i}(F) +  \phi_{i}(G)$.

\textbf{Efficiency:}  The values add up to the difference in value between the grand coalition and the empty coalition:  $\sum_{i\in \groundset} \phi_i(F) = F(\groundset) - F(\emptyset)$.

\begin{remark}[The notion of ``marginalism'']
Specifically,  the marginalism axiom  was first mentioned by  \cite[Equation 7 on page 70]{young1985monotonic}, where it was called the ``independence'' condition; Following \cite[page 121]{chun1989new}, it was formally called the  ``marginality''. This axiom requires a player's payoffs to depend only on his own marginal contributions -- whenever they remain unchanged, his payoffs should be unaffected.
\end{remark}

\begin{algorithm}[htbp]
	\caption{\algname{Naive Mean Field for calculating the \varindex}
        }\label{alg_naive_mf}
	\KwIn{A cooperative game $(\groundset,  F(S))$ with $n$ players. Initial marginal $\x^0\in [0, 1]^n$;   \#epochs $K$.
	}
	\KwOut{The \varindex $\s^* =  \sigma^{-1}(\x^*)$}
	{$\x \leftarrow \x^0$\;}
	\For{epoch from $1$ to $K$}{
	\For{$i = 1 \rightarrow n$}{
		{ let $v_i$  be the player being operated\;}
		{$x_{v_i} \leftarrow \sigma(\nabla_{v_i} \multi(\x)/T) = \bigl(1+ \exp(-
    \nabla_{v_i} \multi(\x)/T\bigr)^{-1}$ \;}
	}
	}
	{$\x^* \leftarrow \x$\;}
\end{algorithm}

\section{The Naive Mean Field Algorithm}
\label{app_alg}

The naive mean field algorithm is one of the most classical algorithm for mean field inference.  It is summarized in \cref{alg_naive_mf}.

\section{Proof of Recovering Classical Criteria}
\label{app_recoving_proof}

For Banzhaf value,  by comparing its definition in \cref{def_banzhaf} with \cref{eq_partial_derivative} it reads,
\begin{align}
\banzhaf_i  = \sum\nolimits_{S\subseteq \groundset - i} [F(S + i) - F(S)] \frac{1}{2^{n - 1}}
& = \nabla_i \multi(0.5*\mathbf{1})  =  T \sigma^{-1}(\mfi(0.5*\mathbf{1}; 1) ),
\end{align}
which is the 1-step variational value initialied at $0.5*\mathbf{1}$.

For Shapley value, according to \cite{grabisch2000equivalent}, here we prove a stronger conclusion regarding the generalization of Shapley value: Shapley interaction index, which is defined for any coalition $S$:
\begin{align}
    \shapley_S = \sum_{T\subseteq \groundset - S} \frac{(n-|T|-|S|)!|T|!}{(n-|S|+1)!}\sum_{L\subseteq S} (-1)^{|S|-|L|} F(L+T).
\end{align}

Given \cite{hammer2012boolean}, we have a second form of the multilinear extension as:
\begin{align}
    \multi(\x) = \sum_{S\subseteq \groundset} a(S) \prod_{i\in S} x_i, \x\in [0, 1]^n,
\end{align}
where $a(S) := \sum_{T\subseteq S} (-1)^{|S|-|T|}F(T)$ is the Mobius transform of $F(S)$.

Then, one can show the $S$-derivative of $ \multi(\x) $ is (suppose $S= \{i_1, ..., i_{|S|}\}$),
\begin{align}
    \Delta_S \multi(\x) := \frac{\partial^{|S|} \multi(\x) }{\partial x_{i_1}, ..., \partial x_{i_{|S|}} }  = \sum_{T \supseteq S} a(T) \prod_{i\in T- S} x_i.
\end{align}
So,
\begin{align}
    \Delta_S \multi(x\textbf{1}) = \sum_{T \supseteq S} a(T) x^{|T|-|S|}.
\end{align}
Then it holds,
\begin{align}
   \int_0^1 \Delta_S \multi(x\textbf{1}) dx = \int_0^1  \sum_{T \supseteq S} a(T) x^{|T|-|S|} dx \\
   =   \sum_{T \supseteq S} a(T) \int_0^1 x^{|T|-|S|} dx \\
   = \sum_{T \supseteq S} a(T) (|T|-|S| + 1)^{-1}
\end{align}
According to \cite{grabisch1997k}, we have $\shapley_S = \sum_{T \supseteq S} a(T) (|T|-|S| + 1)^{-1}$,  then we reach the conclusion:
\begin{align}
    \shapley_S  = \int_0^1 \Delta_S \multi(x\textbf{1}) dx.
\end{align}
When $|S|=1$, we recover the conclusion for Shapley value.

\section{Proof of \cref{thm_axiom}}
\label{app_proof_thm1}

\restattheoremone*

\begin{proof}[Proof of \cref{thm_axiom}]
In step $k$, we know that the value to player $i$ is:

\begin{align}\label{eq_sym_1}
T \sigma^{-1} (\mfi(\x; k))_i  =  \sum_{S\subseteq
	\groundset - i  }\ [F(S+ i) - F(S)]   \prod_{j\in
	S} x_j \prod_{j'\in \groundset -  S - i}
(1-x_{j'})
\end{align}
For the \textbf{null player} property,  since $F(S+i) = F(S)$ always holds, it is easy to see that $T \sigma^{-1} (\mfi(\x; k))_i = 0$ holds for all $i\in \groundset$.

Now we will  show that the \textbf{symmetry} property holds.  The value to player $i'$ is:
\begin{align}\label{eq_sym_2}
T \sigma^{-1} (\mfi(\x; k))_{i'}  =  \sum_{S'\subseteq
	\groundset - i'  }\ [F(S'+ i') - F(S')]   \prod_{j\in
	S'} x_j \prod_{j'\in \groundset -  S' - i}
(1-x_{j'})
\end{align}
Now let us compare different terms in the summands of \cref{eq_sym_1} and \cref{eq_sym_2}. We try to match the summands one by one.   There are two situations:

\textbf{Situation I:} For any  $S\subseteq \groundset - i-i'$, we choose $S'=S$.

In this case  we have $F(S+ i) - F(S) = F(S'+ i') - F(S')$.  For the products of $\x$ we have:
\begin{align}
    & \prod_{j\in
	S} x_j \prod_{j'\in \groundset -  S - i}(1-x_{j'}) - \prod_{j\in
	S'} x_j \prod_{j'\in \groundset -  S' - i}
(1-x_{j'}) = \\
& \prod_{j\in
	S} x_j \prod_{j'\in \groundset -  S - i-i'} (1-x_{j'}) [(1 - x_{i'}) - (1-x_i)]
\end{align}
We know that $x_{i'} = x_i$ holds from step 0, by simple induction, we know that $x_{i'} = x_i$ holds for step $k$ as well.
So in this situation, the summands equal to each other.

\textbf{Situation II:} For any $S = A+i'$, we choose  $S' = A+i$, where $A\subseteq \groundset - i-i'$.  In this case, it still holds that $F(S+ i) - F(S) = F(S'+ i') - F(S')$.
For the products of $\x$ we have:
\begin{align}
    & \prod_{j\in
	S} x_j \prod_{j'\in \groundset -  S - i}(1-x_{j'}) - \prod_{j\in
	S'} x_j \prod_{j'\in \groundset -  S' - i'}
(1-x_{j'}) = \\
& \prod_{j\in
	A} x_j \prod_{j'\in \groundset -  A - i-i'} (1-x_{j'}) [x_{i'} - x_i].
\end{align}
Again, by the simple induction, we know that $x_{i'} = x_i$ holds for step $k$.

The above two situations finishes the proof of symmetry.

For the \textbf{marginalisim} axiom,  one can see that
the  update step for the two games are identical, and it is easy
 to deduce that they produce exactly the same trajectories, given that they have the same initializations.

\end{proof}

\section{Miscellaneous Results in \cref{sec_decoupling}}
\label{append_misc}

\markchange{
\paragraph{Gradient of entropy  in   \cref{subsec_algorithms}}
Note that  $q$ is a fully factorized product  distribution
$q(S; {\x}):= \prod_{i\in S}x_i \prod_{j\notin S}(1-x_j), \x\in
[0,1]^n$, so its entropy $\entropy{q(S; \x)} $ can  be written as
the sum of entropy of $n$ independent Bernoulli distributions.
And the entropy of one Bernoulli distribution with parameter $x_i$ is
$$ -x_i \log x_i - (1-x_i) \log (1 - x_i) $$}

\markchange{
So we have,
\begin{align}
    \nabla_i \entropy{q(S; \x)} &  =  \nabla_i  \sum_{i=1}^n [-x_i \log x_i - (1-x_i) \log (1 - x_i)]  \\
    &  = \nabla_i  [-x_i \log x_i - (1-x_i) \log (1 - x_i)]  \\
    & = \log \frac{1-x_i}{x_i}
\end{align}}

\markchange{
\textbf{Satisfying more axioms is not essential for valuation problems.} Notably, in cooperative game theory, one line of work is to seek for solution concepts that would satisfy more axioms.
However, for valuation problems in machine learning, this is arguably not essential.  For example, similar as what \cite{ridaoui2018axiomatisation} argues,  efficiency   does not make sense for certain games.}

\markchange{
For a simple illustration,  let us consider a  voting game from a  classification model with 3 binary features $\x\in \{0, 1\}^3$ with weights $\w = [2, 1, 1]^\trans$: $f(\x) := \mathbbm{1}_{\{\w^\trans \x \geq 3 \}}$.  Now we are trying to find the valuation of each feature in $\groundset = \{x_1, x_2, x_3\}$.
Naturally, the value function in the corresponding voting game shall be $F(S) = f(\x_S)$ where $\x_S$ means setting the coordinates of $\x$ inside $S$ to be 1 while leaving others to be 0.
In this game let us count  how many times each feature could flip the classification result: for feature $x_1$, there are three situations: $F(\{1,2\}) - F(\{2\})$,   $F(\{1,3\}) - F(\{3 \})$ and
$F(\{1,2, 3\}) - F(\{2, 3 \})$; for feature $x_2$, there are one situation: $F(\{1,2\}) - F(\{1\})$; for feature $x_3$, there are one situation: $F(\{1,3\}) - F(\{1\})$. Then the voting power (or valuation) of each feature shall follows a $3:1:1$ ratio.
By simple calculations, one can see that the Banzhaf values of the three features are $\frac{3}{4}, \frac{1}{4}, \frac{1}{4}$, which is consistent with the ratio of the expected voting power. However, the Shapley values of them are $\frac{4}{6}, \frac{1}{6}, \frac{1}{6}$, which is not consistent due  to satisfying the efficiency axiom.  }

\markchange{
By the above example we are trying to explain that for valuation problems, satisfying more axioms is not necessary, sometimes even does not make sense. Whether more axioms shall be considered and which sets of them shall be added really depend on the specific scenario, which will be left for important future work. }

\paragraph{The ``one-shot sampling trick'' to accelerate \cref{alg_mfi_ga}.}

Indeed,  \varindex needs 5 to 10 iterations to converge. In each iteration
one has to evaluate the gradient of multilinear extension  $\nabla f_{\text{mt}}^F(\mathbf{x})$, which needs MCMC sampling to estimate the exponential sum.

Here we suggest a ``one-shot sampling trick'',  when it is expensive to evaluate the value function  $F(S)$.
This trick could reuse the sampled values in each iteration, such that Alg. 1 could run with the similar cost as calculating  Banzhaf values.

The one-shot sampling trick is built upon one formulation taken from \cref{eq_partial_derivative}:
$$\nabla_i f_{\text{mt}}^F(\mathbf{x}) =  \sum_{S\subseteq
	N - i  }\ [F(S+ i) - F(S)]   q(S; ({\mathbf x | x_{i}\gets 0}))$$
For coordinate $i$ of $\nabla f_{\text{mt}}^F(\mathbf{x})$,   we can firstly sample $m$ coalitions
uniformly randomly from $2^{N-i}$, and evaluate their marginal contributions.  Then in each of the following iterations, one could reuse the one-shot sampled
marginal contributions to estimate the gradient according to the above equation.  In this way, we could make the cost of multi-step running of \cref{alg_mfi_ga}
similar as that of Banzhaf value in terms of the number of $F(S)$ evaluations.
Compared to the original iterative sampling from $q$, this one-shot sampling
might come with a variance-complexity tradeoff, for which we will explore as a future work.

\section{More Configuration Details and Experimental Results}
\label{appd_more_exps}

\subsection{Synthetic Experiments on Submodular (FLID) Games}
\label{exp_syn}

\begin{figure}[htbp]
	\centering
	\includegraphics[width=1.01\textwidth]{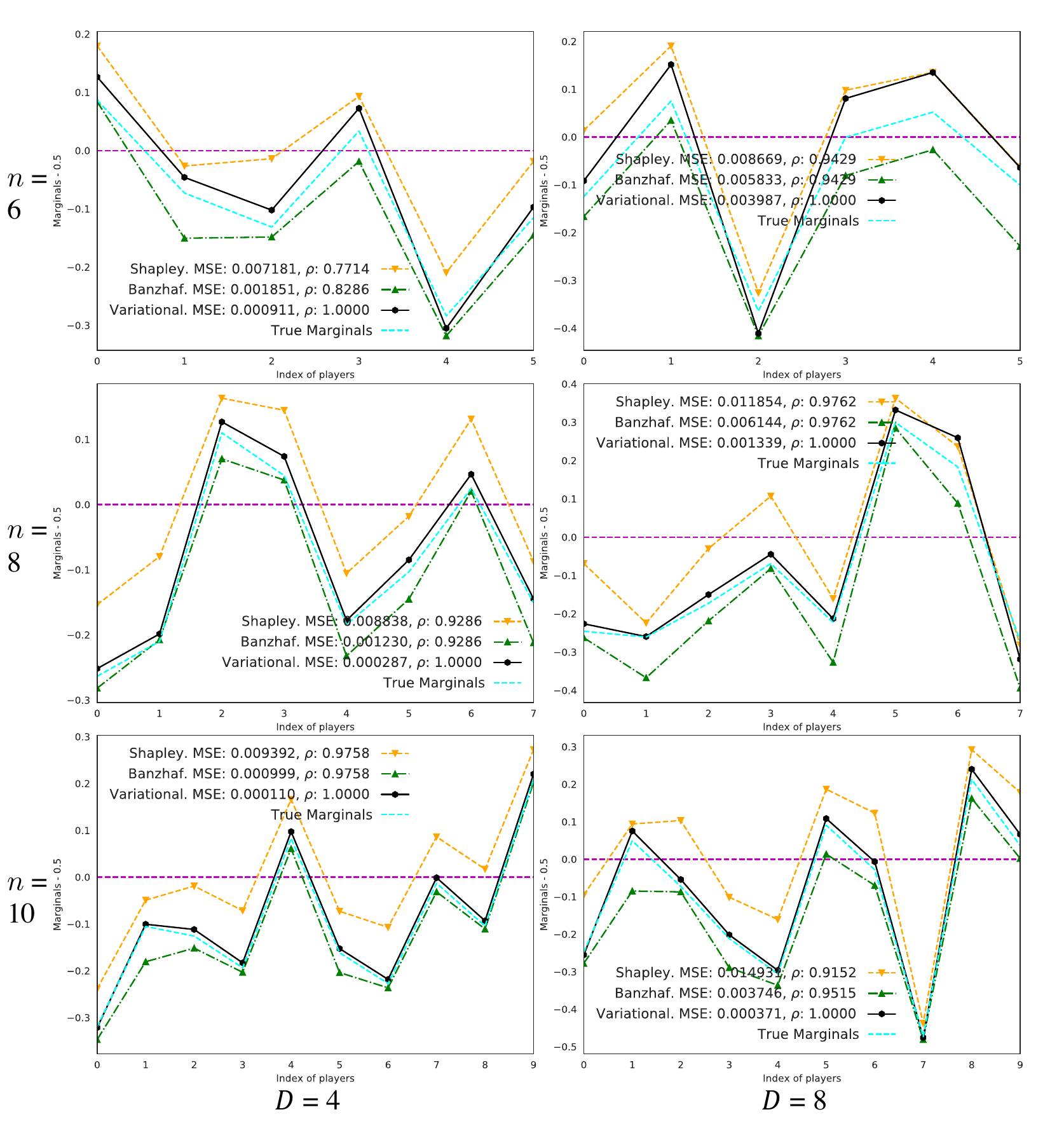}
	\caption{Comparison of different importance measures for FLID games. Hidden dimensions: First column $D$=4, second column $D$=8. $n$ denotes \# of players. Vertical lines means (marginals - 0.5) since we would like to clearly show positive players (marginal > 0.5) and negative players (marginal < 0.5).}
	\label{fig_flid}
\end{figure}

Here we define a synthetic game with the value function as
a FLID   \citep{Tschiatschek16diversity}  objective $F(S)$, which is a diversity boosting model satisfying the submodularity property.
We know that its multilinear extension admits polynomial  time algorithms \citep{bian2019optimalmeanfield}.  Let
$\BW\in \R_+^{|\groundset|\times D }$ be the weights, each row
corresponds to the latent representation of an item, with  $D$ as the
dimensionality. Then
\begin{align}
  F(S) := & \sum\nolimits_ {i\in S} u_i + \sum\nolimits_{d=1}^{D} (
            \max_{i\in S} W_{i,d} - \sum\nolimits_{i\in S} W_{i,d} )  =  \sum\nolimits_ {i\in S} u'_i + \sum\nolimits_{d=1}^{D}
            \max_{i\in S}W_{i,d},
\end{align}
which models both coverage and diversity, and
$u'_i = u_i - \sum_{d=1}^{D} W_{i,d}$.
In order to test the performance of the proposed variational objectives,
we consider  small synthetic games with 6, 8 and 10 players such that the ground truth marginals can be computed exhaustively.
We would like to compare with the true marginals $p(i\in \BS)$ since they represent
the probability that player $i$ participates in all coalitions, which is hard to compute in general.  The distance to the true marginals is also a natural measure of the decoupling error as defined in \cref{def_decoupling_error}.
We apply a sigmoid function to Shapley value and Banzhaf value in order to translate them to  probabilities.
We calculate the mean squared error (MSE) and Spearman's rank correlation ($\rho$) to the ground truth marginals $p(i\in \BS)$ and  report them in the figure legend.
\cref{fig_flid} collects the figures, one can see that the \varindex clearly obtains better performance in terms of MSE and Spearman's rank correlation compared to
the one-point solutions (Shapley value and Banzhaf value) in all experiments.

\subsection{Details of the Models for Feature Valuations}

For xgboost, the train  accuracy is   0.8934307914376094, the specific configuration  with the xgboost package is:

\lstset{language=Python}
\lstset{basicstyle=\footnotesize}
\begin{lstlisting}
XGBClassifier(base_score=0.5, booster='gbtree', colsample_bylevel=1,
              colsample_bynode=1, colsample_bytree=1, gamma=0, gpu_id=-1,
              importance_type='gain', interaction_constraints='',
              learning_rate=0.300000012, max_delta_step=0, max_depth=6,
              min_child_weight=1, missing=nan, monotone_constraints='()',
              n_estimators=100, n_jobs=56, num_parallel_tree=1,
              random_state=0, reg_alpha=0, reg_lambda=1,
              scale_pos_weight=1, subsample=1, tree_method='exact',
              validate_parameters=1, verbosity=None)
\end{lstlisting}

We used the sklearn package for the logistic regression and MLP classifiers.
For logistic regression, the accuracy is 0.8418967476428857, the specific configuration is:
\begin{lstlisting}
LogisticRegression(random_state=0, solver="liblinear", C=0.5)
\end{lstlisting}

For the MLP, its accuracy is 0.8614600288688923, and the  configuration is:
\begin{lstlisting}
MLPClassifier(random_state=0, max_iter=300,
              learning_rate_init=0.002,
              hidden_layer_sizes=(50,50))
\end{lstlisting}

\subsection{More Results on Feature Valuations}

In this part we provide more results on feature removal  in \cref{fig_appen_feature_removing1,fig_appen_feature_removing2}.  \cref{fig_appen_feature_removing1} shows
similar behavior as that shown in the main text.
\cref{fig_appen_feature_removing2} provides not very distinguishable  results of the three criteria.

\begin{figure}[h!]
	\centering
	\includegraphics[width=1.02\textwidth]{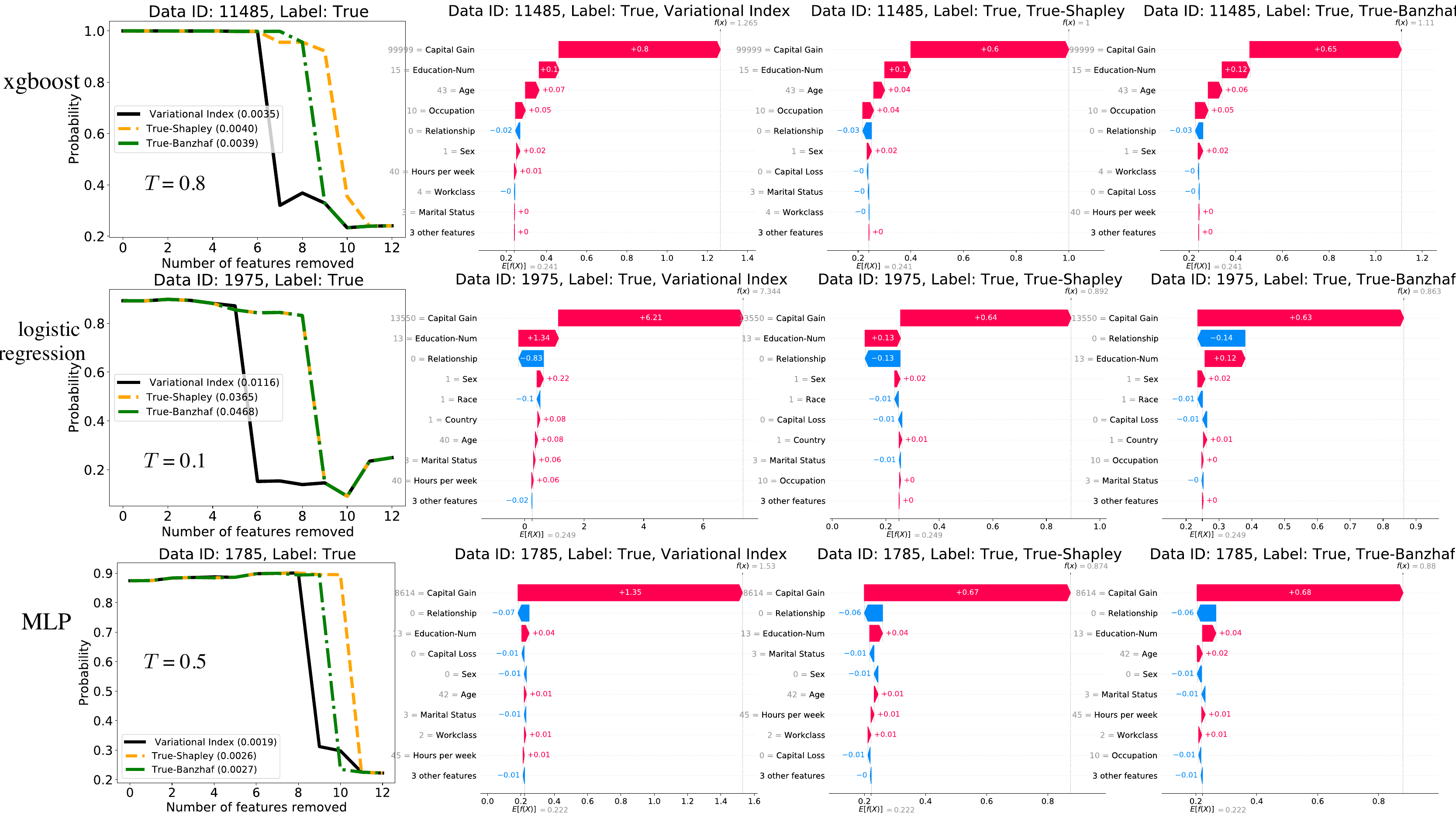}
	\caption{ First column: Change of predicted probabilities when removing features. The \emph{decoupling error} is included in the legend.   Last three columns: waterfall plots of  feature importance  from \varindex, Shapley and Banzhaf.}
	\label{fig_appen_feature_removing1}
\end{figure}

\begin{figure}[h!]
	\centering
	\includegraphics[width=1.02\textwidth]{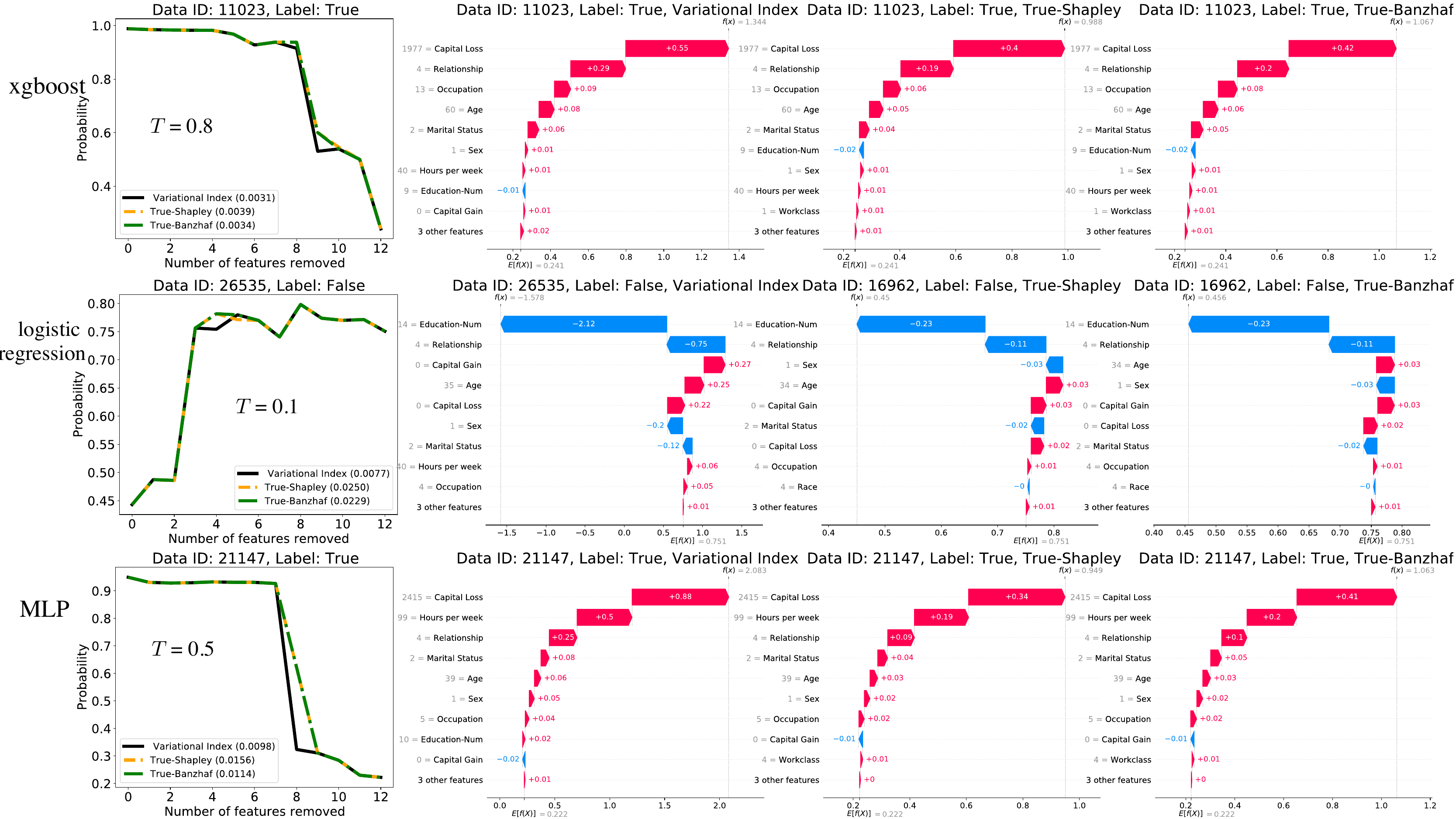}
	\caption{Not very distinguishable results. First column: Change of predicted probabilities when removing features. The \emph{decoupling error} is included in the legend.   Last three columns: waterfall plots of  feature importance  from \varindex, Shapley and Banzhaf.}
	\label{fig_appen_feature_removing2}
\end{figure}

\subsection{More Average Results on Feature Valuations}

We provide additional statistical results with the MLP model (\cref{fig_stats_mlp}) and logistic regression model (\cref{fig_stats_lr}).
Meanwhile, we also put the full statistics on the xgboost model in \cref{fig_stats_xgboost} since in the main text the table of the samples with ``False'' grouthtruth label is skipped due to space limit.
From \cref{fig_stats_mlp} one can observe that  \varindex induces different rankings of the features compared to Shapley and Banzhaf: \varindex ranks
``Marital Status'' the second, while Shapley and Banzhaf put it in the third location.

It is also very interesting to see that the logistic regression model (with the lowest training accuracy among the three models, shown in \cref{fig_stats_lr}) provides different ranking for the first two features compared to MLP and xgboost. For the samples with ``True'' groundtruth labels, ``Education-Num'' is the first important feature for the logistic regression model, while ``Captical Gain'' was ranked first for the MLP and xgboost.

\begin{figure}[h!]
	\centering
	\includegraphics[width=1.02\textwidth]{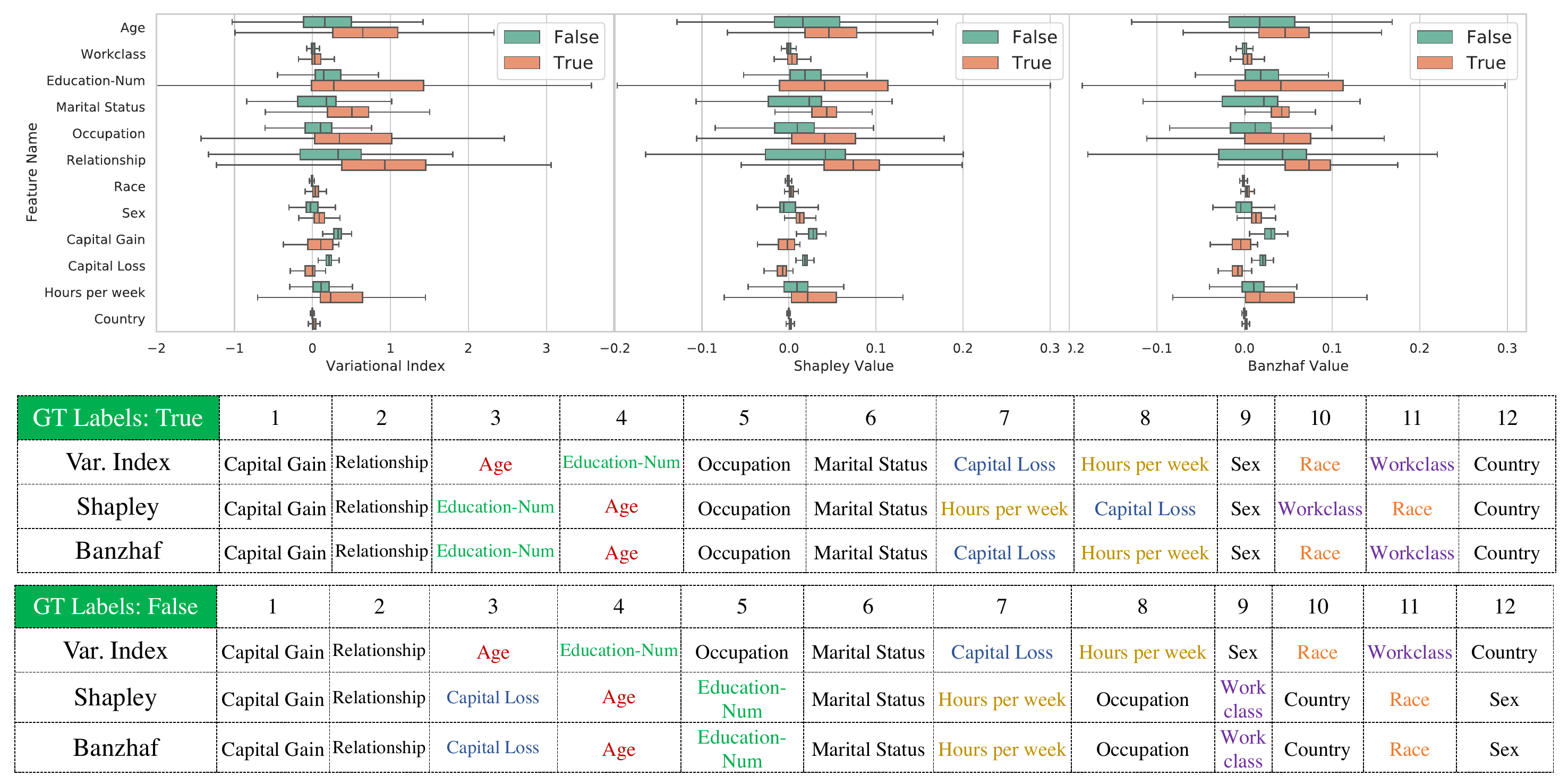}
	\caption{Full statistics on valuations with the \emph{xgboost} classifier (in the main text the last row is skipped due to space limit). First row: box plot of valuations
	returned by the three criteria.
	We always consider the predicted probability of the ground truth label. ``True'' means the samples with positive ground truth  label and ``False'' means with the negative ground truth label. Second and third rows: Average ranking of the 12 features. Colored texts denote different rankings among the three criteria.}
	\label{fig_stats_xgboost}
\end{figure}

\begin{figure}[h!]
	\centering
	\includegraphics[width=1.02\textwidth]{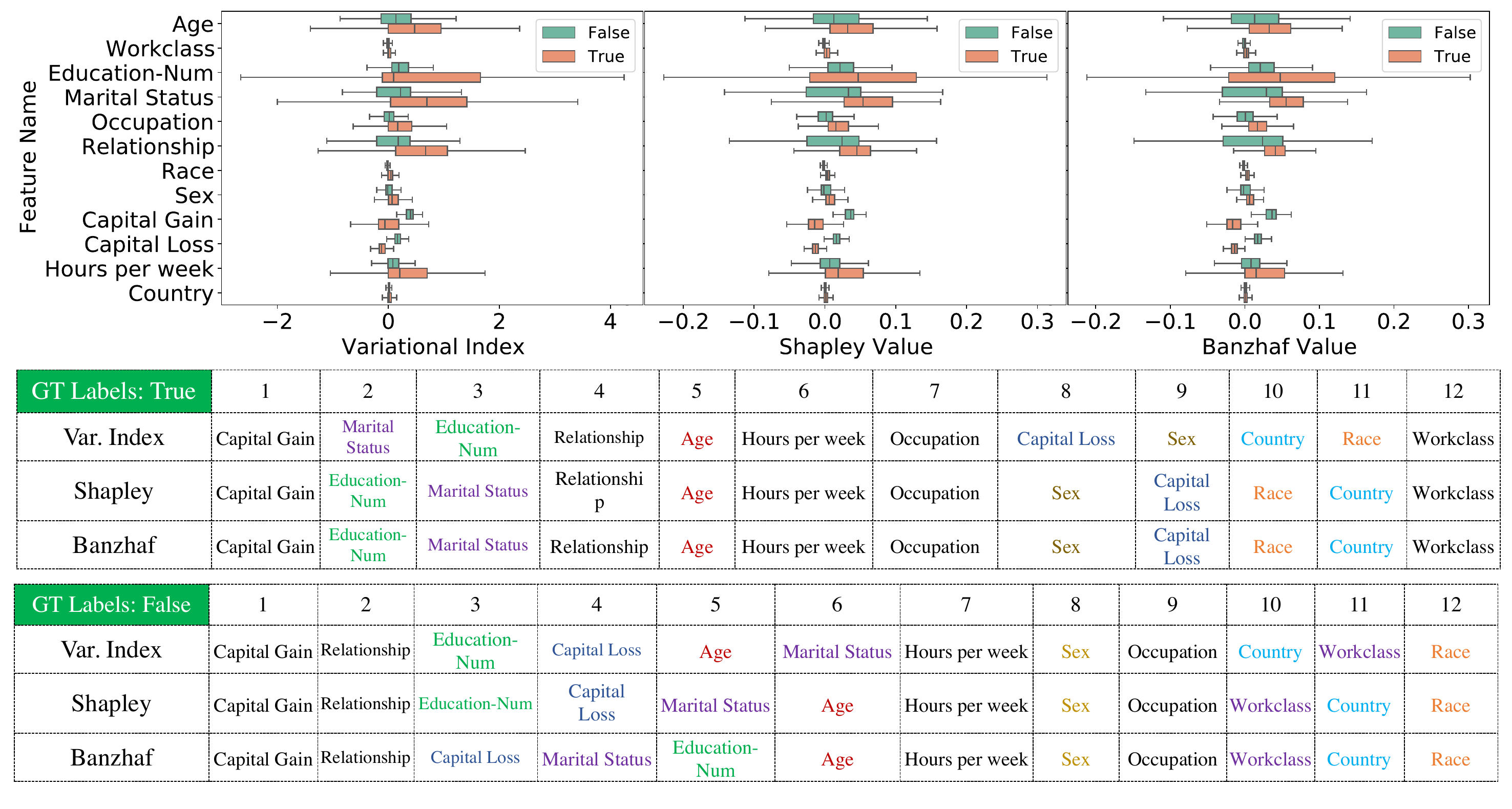}
	\caption{Statistics on valuations with the \emph{MLP} classifier. First row: box plot of valuations
	returned by the three criteria.
	We always consider the predicted probability of the ground truth label. ``True'' means the samples with positive ground truth  label and ``False'' means with the negative ground truth label. Second and third rows: Average ranking of the 12 features. Colored texts denote different rankings among the three criteria.}
	\label{fig_stats_mlp}
\end{figure}

\begin{figure}[h!]
	\centering
	\includegraphics[width=1.02\textwidth]{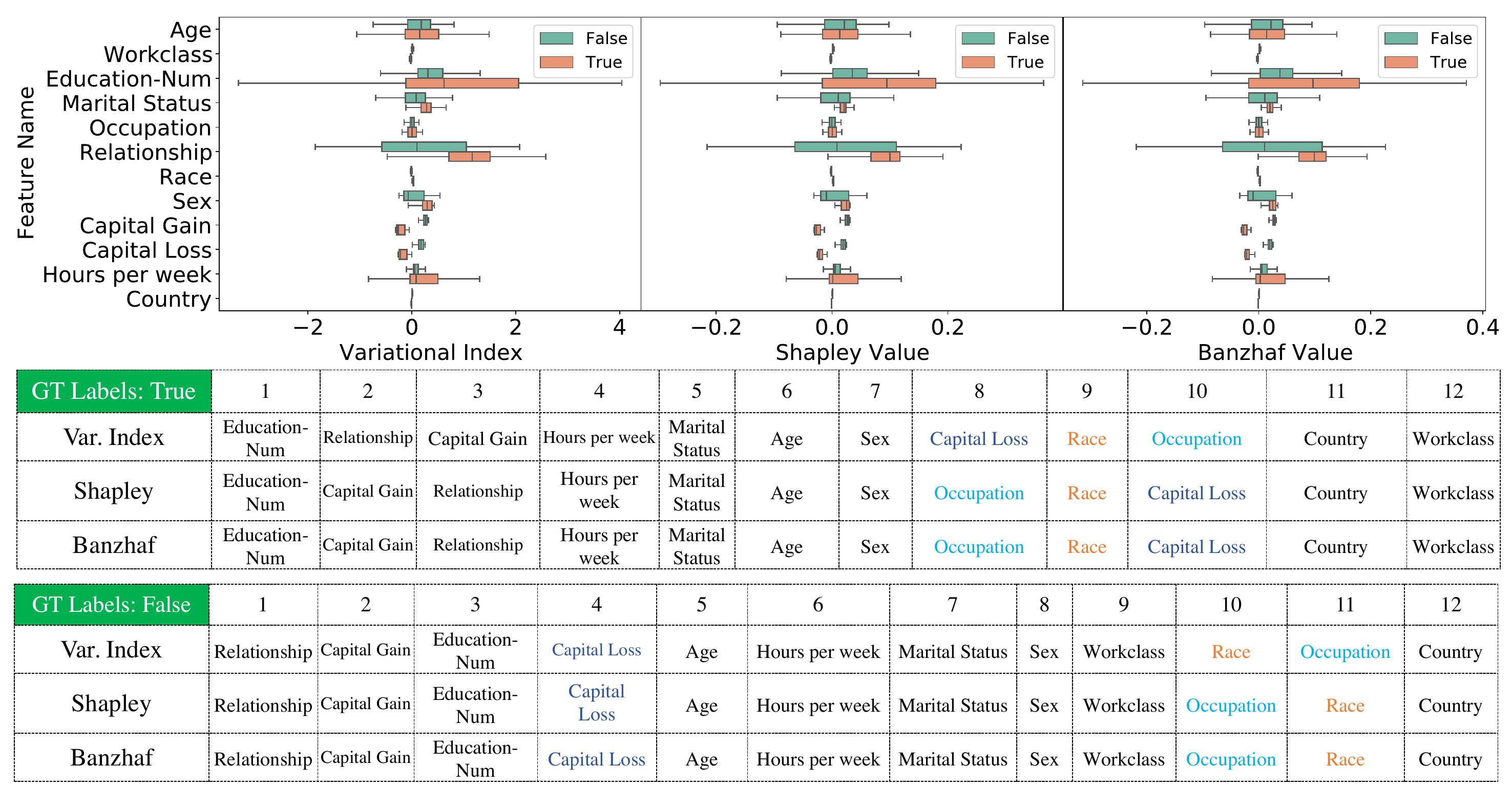}
	\caption{Statistics on valuations with the \emph{logistic regression} classifier. First row: box plot of valuations
	returned by the three criteria.
	We always consider the predicted probability of the ground truth label. ``True'' means the samples with positive ground truth  label and ``False'' means with the negative ground truth label. Second and Third rows: Average ranking of the 12 features. Colored texts denote different rankings among the three criteria.}
	\label{fig_stats_lr}
\end{figure}

\subsection{Experiments with More Players}
\label{apendix_large_players}

Furthermore, we  experiment with a bit larger number of players ($n=80$) using MCMC sampling to approximate the partial derivative in \cref{eq_partial_derivative}. The sampling based approximation works pretty fast.  \cref{tab_large}  shows  the top 15 ranked players returned by our method, Shapley value and Banzhaf value for a synthetic data valuation problem. Note that the ranking of  \varindex is more similar to that of Banzhaf value than that of Shapley value.

\begin{table}[h!]
 \caption{Indices of the top 15 ranked players returned by different methods.}
\label{tab_large}
\centering
 \begin{tabular}{r c c c c c c c c c c c c c c c }
 \hline
 Rank of players&	1&	2&	3&	4&	5&	6&	7&	8&	9&	10&	11&	12&	13&	14&	15 \\
 \hline
Variational Index&	9&	13&	11&	3&	54&	7&	18&	36&	32&	42&	46&	40&	6&	18&	23 \\
Banzhaf value	&9&	13	&11&	3&	54&	18&	7&	32&	46&	36&	2&	27&	40&	17&	10 \\
Shaplay Value&	52&	33&	27&	1&	4&	58&	32&	14&	42&	46&	40&	6&	18&	23&	47\\
 \hline
 \end{tabular}
\end{table}

\subsection{Effect of Number of Samples in MCMC Sampling}
\label{apendix_tradeoff}

Here we illustrate the accuracy sampling tradeoff when estimating the
gradient of multilinear extension using MCMC sampling.  The results is shown
in \cref{fig_tradeoff}.  One can observe that with more samples, \cref{alg_mfi_ga} will converge faster (in fewer number of epochs).

\begin{figure}[h!]
	\centering
	\includegraphics[width=.5\textwidth]{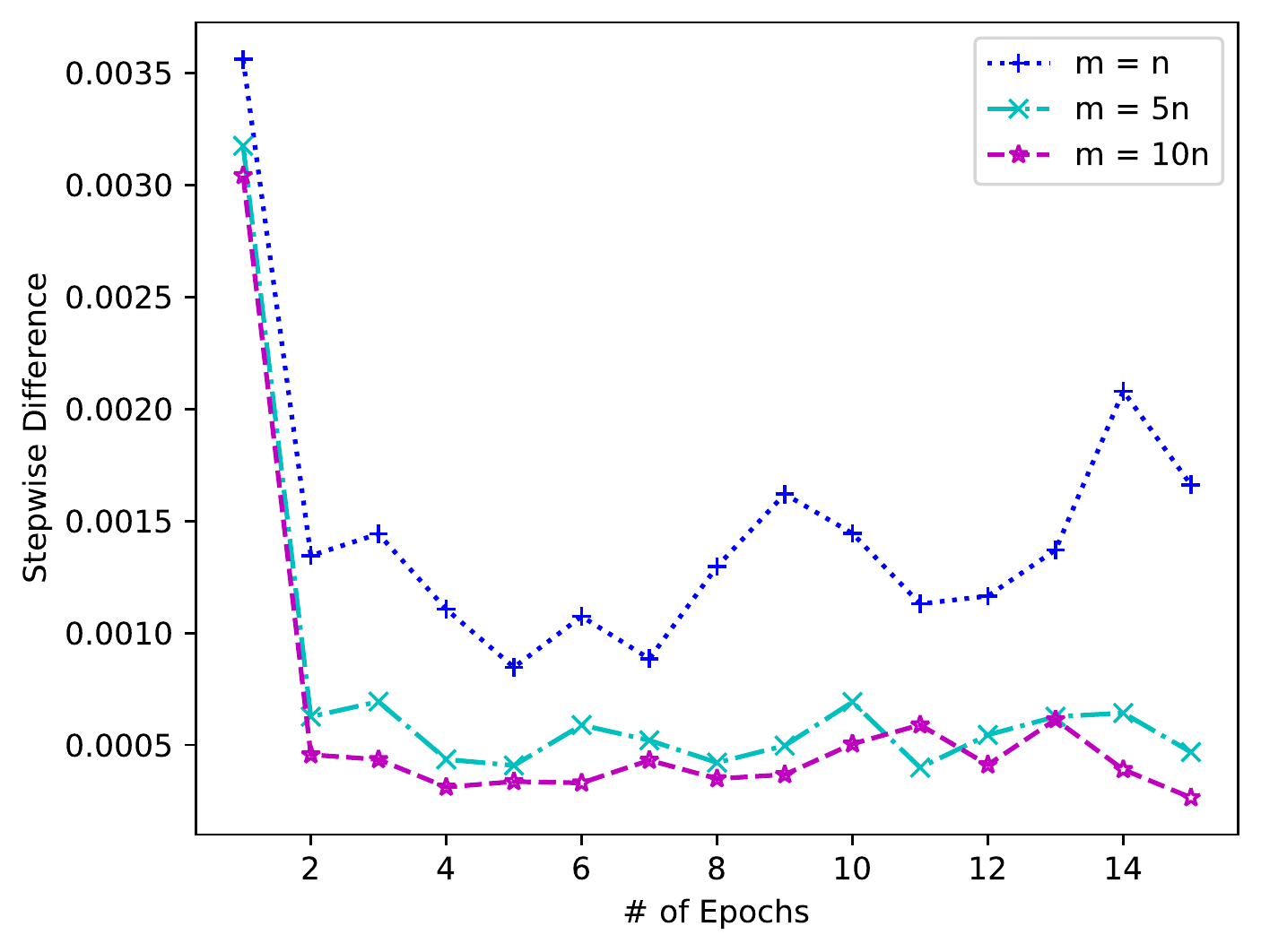}
	\caption{Stepwise difference with difference number of samples $m$ when  estimating the gradient of multilinear extension using MCMC sampling.  The three curves:  $m=n$, $m=5n$, $m=10n$.}
	\label{fig_tradeoff}
\end{figure}

\end{document}